\documentclass[accepted]{uai2026} 
                        
\usepackage[american]{babel}

\usepackage[dvipsnames]{xcolor}
\newcommand{\first}[1]{\textbf{\textcolor{red}{#1}}}
\newcommand{\second}[1]{\textbf{\textcolor{violet}{#1}}}
\newcommand{\third}[1]{\textbf{\textcolor{black}{#1}}}

\usepackage{algorithm}
\usepackage{algorithmic}
\usepackage{amsthm}

\usepackage{amsmath,amsfonts}

\def\eqref#1{equation~\ref{#1}}
\def\Eqref#1{Equation~\ref{#1}}

\def\1{\bm{1}}

\DeclareMathAlphabet{\mathsfit}{\encodingdefault}{\sfdefault}{m}{sl}
\SetMathAlphabet{\mathsfit}{bold}{\encodingdefault}{\sfdefault}{bx}{n}

\newcommand{\bfI}{{\bf I}}

\newcommand{\bfS}{{\bf S}}

\newcommand{\bfy}{{\bf y}}
\newcommand{\bfu}{{\bf u}}
\newcommand{\bfq}{{\bf q}}
\newcommand{\bfp}{{\bf p}}

\newcommand{\bfv}{{\bf v}}

\newcommand{\bfz}{{\bf z}}

\newcommand{\bfepsilon}{{\boldsymbol \epsilon}}
\newcommand{\bftheta}{{\boldsymbol \theta}}

\newtheorem{theorem}{Theorem}[section]
\newtheorem{lemma}[theorem]{Lemma}
\newtheorem{definition}{Definition}
\newtheorem{example}{Example}

\newtheorem{comment}{Comment}
\newtheorem{assumption}{Assumption}

\usepackage{cleveref}
\usepackage{pdflscape}
\usepackage{soul}

\usepackage{natbib} 
\usepackage{mathtools} 
\usepackage{booktabs} 
\usepackage{tikz} 

\title{Fast and Flexible Probabilistic Forecasting of Dynamical Systems\\ using Flow Matching and Physical Perturbation}

\author[1]{\href{mailto:<siddharth.rout@ubc.ca>?Subject=Your UAI 2026 paper}{Siddharth~Rout}{}}
\author[1]{Eldad~Haber}
\author[2]{ St\'ephane~Gaudreault}
\affil[1]{%
    University of British Columbia\\
    Vancouver, BC, Canada
}
\affil[2]{%
    Environnement et Changement climatique Canada \\
    Dorval, QC, Canada
}
  
\begin{document}
\maketitle

\begin{abstract}
Learning dynamical systems from incomplete or noisy data is inherently ill-posed, as a single observation may correspond to multiple plausible futures. While physics-based ensemble forecasting relies on perturbing initial states to capture uncertainty, standard Gaussian or uniform perturbations often yield unphysical initial states in high-dimensional systems. Existing machine learning approaches address this via diffusion models, which rely on inference via computationally expensive stochastic differential equations (SDEs). We introduce a novel framework that decouples perturbation generation from propagation. First, we propose a flow matching-based generative approach to learn physically consistent perturbations of the initial conditions, avoiding artifacts caused by Gaussian noise. Second, we employ deterministic flow matching models with Ordinary Differential Equation (ODE) integrators for efficient ensemble propagation with fewer integration steps. We validate our method on nonlinear dynamical system benchmarks, including the Lotka-Volterra Predator-Prey system, MovingMNIST, and high-dimensional WeatherBench data (5.625$^\circ$). Our approach achieves state-of-the-art probabilistic scoring, as measured by the Continuous Ranked Probability Score (CRPS), and physical consistency, while offering significantly faster inference than diffusion-based baselines. 
\end{abstract}

\section{Introduction}\label{sec:intro}

Learning dynamical systems is of paramount importance across scientific and engineering disciplines, ranging from numerical weather prediction (NWP) to finance and biology \citep{brin2002introduction, tu2012dynamical}. Traditionally, predicting the behaviour of these systems relies on explicit mathematical models derived from physical laws. However, these classical approaches are often constrained by computational limits \citep{moin1998direct, wedi2015modelling} and the inherent complexity of chaotic behavior \citep{guckenheimer2013nonlinear}. Crucially, when data is missing, noisy, or sparse, this inverse problem becomes ill-posed: a single observation may correspond to a distribution of plausible future states rather than a unique deterministic outcome. In such scenarios, probabilistic forecasting is essential for robust decision-making \citep{reich2015probabilistic}.

Machine learning (ML) has emerged as a powerful alternative for learning complex spatiotemporal patterns without requiring explicit physical equations \citep{Ghadami2022}. Successes in high-resolution weather forecasting, such as FourCastNet \citep{pathak2022fourcastnet}, GraphCast \citep{graphcast2023}, and Aurora \citep{bodnar2024aurora}, demonstrate that deep learning can rival traditional numerical solvers in accuracy while significantly reducing inference time.

Despite these advances, most deep learning architectures (ranging from classical RNNs and LSTMs to modern Vision Transformers \citep{liu2021swin, bi2023pangu}) are fundamentally deterministic, mapping a point input to a point output. While effective for closed systems with perfect data, these models fail to capture the intrinsic uncertainty of open systems. To address this, recent works have adapted diffusion models and stochastic interpolants to learn mappings from a point to a distribution \citep{price2025probabilistic, chen2024probabilistic}. These methods typically rely on learning high-dimensional Stochastic Differential Equations (SDEs) during inference. Although mathematically rigorous, SDE-based sampling is computationally expensive and scales poorly to real-time applications due to high training and inference costs. 

\paragraph{Motivation.} 
To mitigate the computational cost of SDEs, we draw inspiration from classical ensemble forecasting \citep{kalnay2003atmospheric}. Instead of learning a stochastic evolution (point-to-distribution) via SDEs, we reformulate the problem as learning a deterministic transport map between distributions (distribution-to-distribution). This reformulation enables us to use Ordinary Differential Equations (ODEs) rather than SDEs, which offers faster and more accurate numerical integration \citep{liu2022flow}.

However, ensemble forecasting introduces a secondary challenge: generating a representative ensemble of initial conditions. The state space of a dynamical system is often a complex, lower-dimensional manifold embedded in a high-dimensional space. Standard techniques (such as adding Gaussian or uniform noise to the input) often produce "unphysical" states that lie off this manifold \citep{kalnay2003atmospheric}. Propagating unphysical initial conditions leads to model drift and unreliable uncertainty estimates.

We propose a unified framework that addresses both computational efficiency and physical consistency challenges. We exploit a continuous normalizing flow \cite[CNF,][]{chen2018neural} based invertible generative technique called Flow Matching \cite[FM,][]{lipman2022flow} to learn the manifold of valid states and sample physically consistent perturbations. These perturbations are then propagated using a deterministic ODE-based flow model.

\subsection{Contributions}
\begin{itemize}
    \item \textbf{Generative perturbation for physically meaningful sampling:} We propose a generative variant of flow matching to perturb high-dimensional complex dynamical states. Unlike Gaussian noise, our method ensures perturbations remain on the data manifold, preserving physical consistency.
    \item \textbf{Efficient uncertainty propagation:} By formulating the forecasting problem as a distribution-to-distribution mapping, we employ ODE-based flow matching. This approach replaces computationally expensive SDEs, enabling faster training and inference while decoupling stochasticity from the dynamics.
    \item \textbf{Flexible uncertainty quantification:} Our approach decouples the perturbation step from the forecasting step, allowing flexible specification of when uncertainty is introduced. The forecasting and uncertainty models can be used either independently or in combination with traditional, state-of-the-art, or physics-based approaches. See \Cref{fig:flexibleUQ} for a graphical abstract.  
    \item \textbf{Empirical validation on diverse benchmarks:} We demonstrate state-of-the-art performance on complex coupled nonlinear systems (Predator-Prey), video prediction (MovingMNIST), and high-dimensional climate modelling (WeatherBench), showing improved uncertainty quantification Continuous Ranked Probability Score (CRPS) over diffusion-based baselines.
\end{itemize}

\begin{figure}[t]
    \centering
    \includegraphics[width=\linewidth]{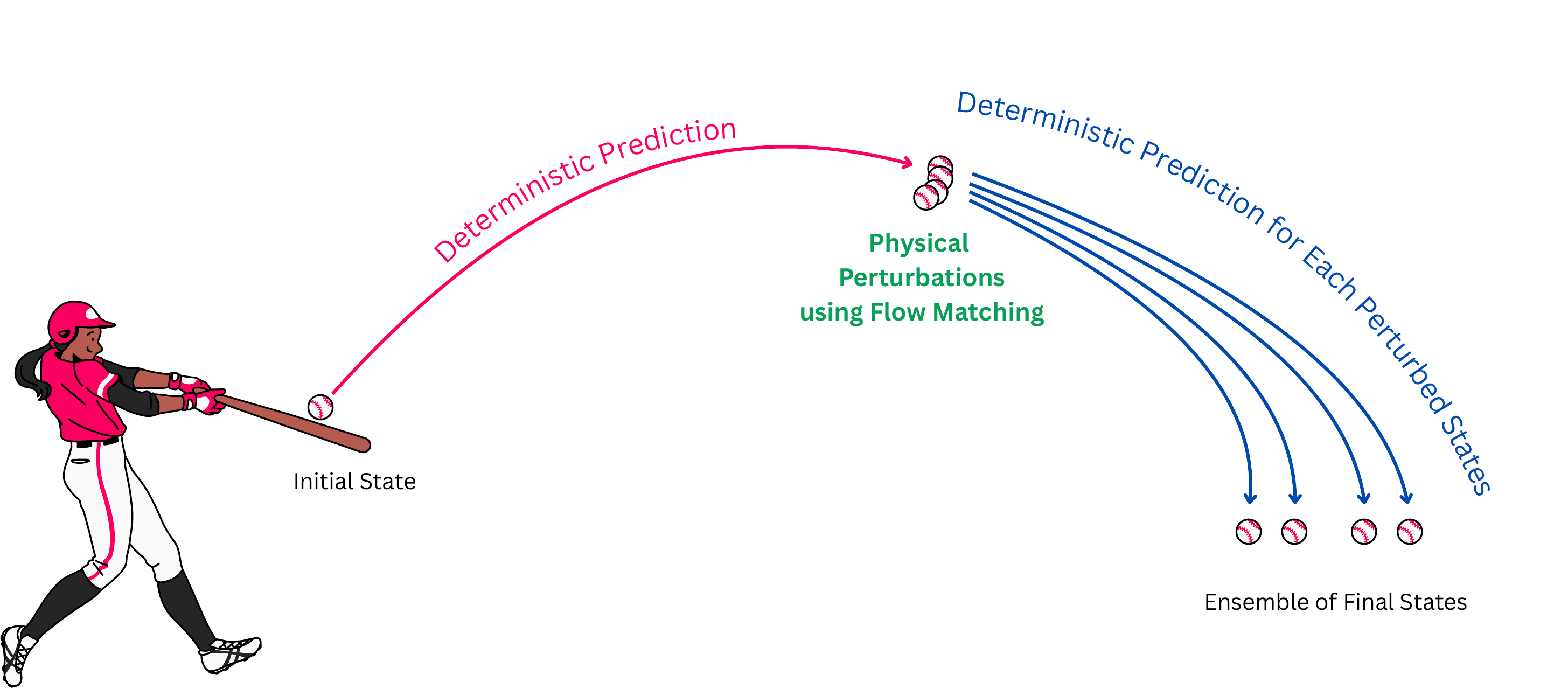}
    \caption{A cartoonified graphical abstract of our proposition in the context of probabilistic forecasting of baseball trajectory after a hit using our technique. Notably, the physical perturbation is generated independently from the deterministic forecasting process.}
    \label{fig:flexibleUQ}
\end{figure}

\section{Mathematical Foundation}

Consider a dynamical system of the form
\begin{eqnarray}
    \label{dynsym}
    \dot \bfy = f(\bfy, t, \bfp).
\end{eqnarray}
Here, $\bfy \in \mathbb{R}^n$ is the state vector and $f$ is a function that depends on state $\bfy$, time $t$, and parameters $\bfp \in \mathbb{R}^k$.
We assume that $f$ is sufficiently smooth and differentiable so that, given an initial condition $\bfy(0)$, one can compute $\bfy(T)$ using some numerical integration method \citep{stuart1998dynamical}. Our goal is to learn the function $f$ given observations $\bfy(t)$. Several approaches can accomplish this \citep{reich2015probabilistic, chattopadhyay2020data, Ghadami2022}. 

We assume $\bfy$ is measured at constant time intervals $T$. In such cases,
\begin{eqnarray}
    \label{eq:dynsymT}
    \bfy(t+T) = \mathcal{F}(\bfy(t), t, \bfp) = \bfy(t) + \int_t^{t+T} f(\bfy(\tau), \tau, \bfp) d\tau.
\end{eqnarray}
Here, $\mathcal{F}$ is the function that integrates the ODE from time $t$ to $t+T$. We can then approximate $\mathcal{F}$ using a sufficiently expressive parameterised function approximator $g$ and solving the optimisation problem
\begin{eqnarray}
    \label{eq:fd}
   \min_{\bftheta}\ \sum_j \left\| g\left(\bfy_j, t_j, \bftheta \right) -  \bfy_{j+1} \right\|_2,
\end{eqnarray}
where $\bftheta$ are learnable parameters in $g$. Note that $g$ is only a surrogate for $\mathcal{F}$, and the parameters $\bftheta$ differ from those of the true dynamical system, $\bfp$. The above scenario falls under the case where we can predict future states from current values of $\bfy$, independent of the time interval. 

\begin{definition}{\bf Closed System.} \label{def1}
Let ${\cal Y}$ be the space of observed data $\bfy$ from a dynamical system ${\cal D}$. We define the system as closed if, given $\bfy(t)$, we can uniquely recover $\bfy(t+T)$ for all $t$ and finite, bounded $T\le\tau$, where $\tau$ is a constant. 
In practice, given the data $\bfy(t)$ and a tolerance $\epsilon > 0$, we can estimate a function $\mathcal{F}$ such that
\begin{eqnarray}
    \label{eq:close}
    \|\bfy(t+T) - \mathcal{F}(\bfy(t), t, \bfp) \|^2 \le \epsilon ~.
\end{eqnarray}
\end{definition}

Definition \ref{def1} implies that we can learn the function $F$ in \Eqref{eq:dynsymT} assuming that we have a sufficient amount of data in sufficient accuracy and an expressive enough neural architecture to approximate $\mathcal{F}(\cdot, \cdot, \cdot)$. In this case, the focus of any ML-based method should be devoted to the appropriate architecture that can express $\mathcal {F}$, perhaps using some of the known structure of the problem.

\begin{definition} \label{def2}{\bf Open System.}
Let ${\cal Y}$ be the space of some observed data $\bfy$ on a dynamical system ${\cal D}$. We say the system is open if, given $\bfy(t)$, we {\bf cannot} uniquely recover $\bfy(t+T)$ for all $t$  and finite $T\le\tau$, where $\tau$ is some constant. That is, there is a constant $\epsilon$ such that
\begin{eqnarray}
\label{eq:open}
   \|\bfy(t+T) - \mathcal{F}(\bfy(t), t, \bfp) \| \ge \epsilon~,
\end{eqnarray}
regardless of data quantity or the complexity of $\mathcal{F}$.
\end{definition}

This concept and its limitations are illustrated using the predator-prey model. This example demonstrates how a seemingly simple dynamical system can exhibit complex behaviour through the interaction of just two interdependent variables. While theoretically deterministic, the system can become practically unpredictable if measured in regions where the predator-prey cycle trajectories are very close; in other words, future states are highly sensitive to certain initial states.

\begin{example} {\bf Predator-Prey Model :}
The predator-prey model \citep[Chapter 3]{murray2007mathematical} is
\begin{eqnarray}
    \label{eq:pp}
{\frac {d\bfy_1}{dt}} = \bfp_1\bfy_1  - \bfp_2 \bfy_1 \bfy_2 \quad {\frac {d\bfy_2}{dt}} = \bfp_3 \bfy_1 \bfy_2  -\bfp_4 \bfy_2
 \end{eqnarray}

\begin{figure}[t]
        \centering
        \includegraphics[width=0.8\linewidth]{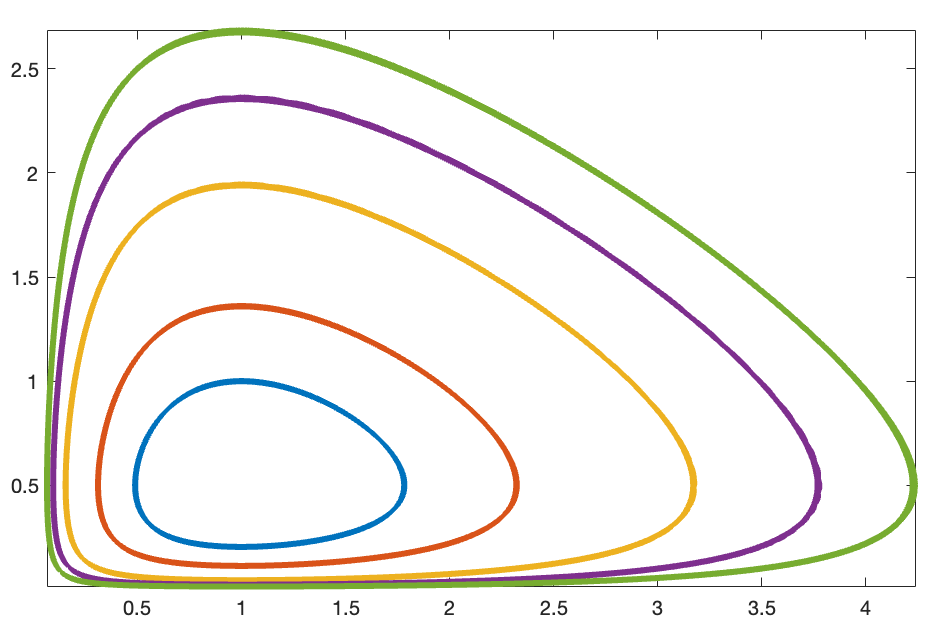}
        \caption{Temporal trajectories for the predator-prey model. Note that the trajectories approach each other closely but do not intersect.}
        \label{fig:volot}
\end{figure}

The temporal trajectories for different initial conditions $\bfy(0)$ and a given physical parameter $\bfp$ are shown in \Cref{fig:volot}, where each initial condition produces a cyclic trajectory shown in a single color. Assuming we record data with sufficient accuracy, the trajectory at any time can be determined by the measurement point at earlier times, thus justifying the model in \Eqref{eq:dynsymT}. Furthermore, since the system is periodic, it is straightforward to formulate a learning problem that can integrate the system over long time periods.

Note that trajectories cluster near the lower-left corner. While trajectories never intersect in theory, they approach one another closely in this region. Consequently, numerical solutions with noisy data may exhibit bifurcation-like behavior near this point. The system is therefore closed if the data $\bfy$ is accurate, yielding the final state as a specific point on the same trajectory. However, the system is open if noise in the data causes later times to be significantly influenced by uncertainties at earlier times, yielding a broad probability distribution of possible final states, see \Cref{ex2}.
\end{example}

\begin{example} \label{ex2} {\bf Predator-Prey Model with Partial or Noisy Data:}
Consider the predator-prey model again, but this time assume we observe only $\bfy_1$ (partial data). Now assume that we know that at $t=0$, $\bfy_1=1$ and our goal is to predict the solution at time $t=200$, which is very far into the future. This is computationally intractable. However, we can perform many simulations where $\bfy_1 = 1$ and $\bfy_2 = \pi(\bfy_2)$ where $\pi$ is some probability density function used to sample $\bfy_2$. For instance, let $\pi(\bfy_2) = U(0,1)$. The results in \Cref{fig:ts} (left) demonstrate that the solutions are not unique. The collection of points obtained at $t = 200$ can be interpreted as samples from a probability density function, which lies on a curve in state space. Our goal is now transformed from learning the function $F(\bfy,t)$ in \Eqref{eq:dynsymT} to learn a probability density function
$\pi(\bfy(T))$.

\begin{figure}[t]
        \centering
        \begin{tabular}{cc}
        \includegraphics[width=0.47\linewidth]{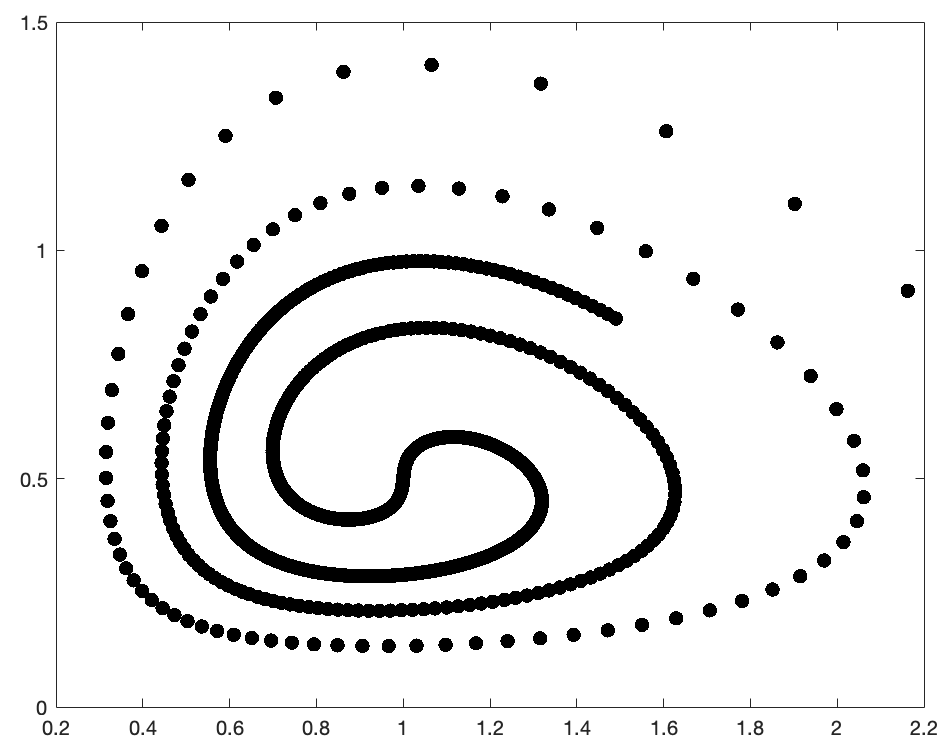} &
        \includegraphics[width=0.47\linewidth]{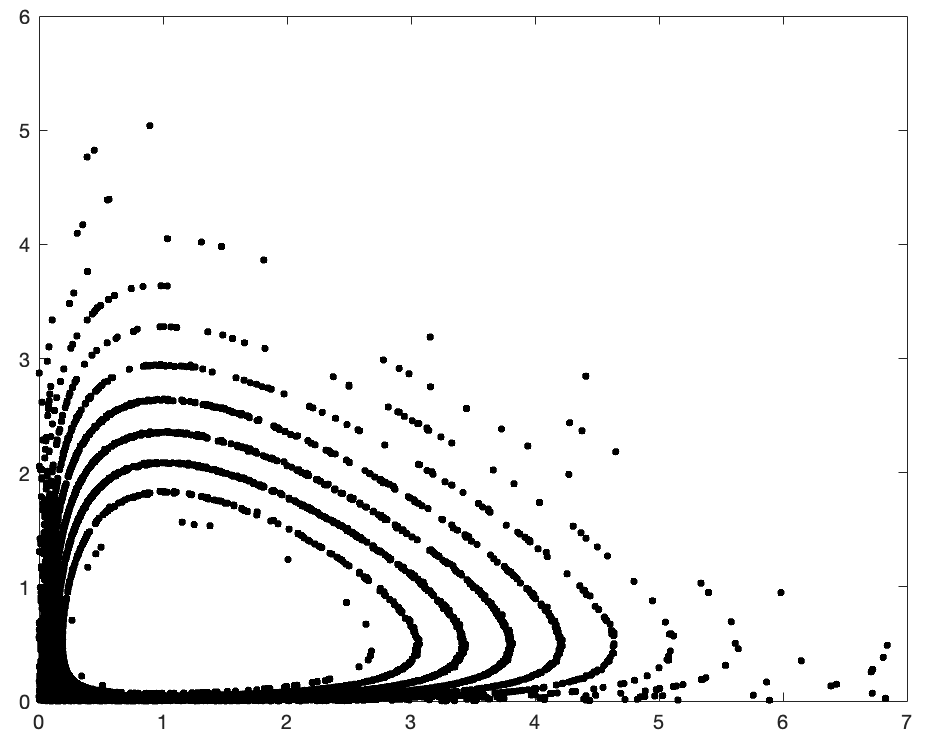}
        \end{tabular}
        \caption{Left: The solution for $\bfy_1(0)=1$ and $\bfy_2(0) \sim U(0,1)$ at $t=200$. Right:The solution for $\bfy(0)=[0.1, 0.3]^{\top} + \epsilon$ where $\epsilon \sim N(0, 0.05 \bfI)$ at $t=200$}
        \label{fig:ts}
\end{figure}

The second case is when the data is noisy. Consider the case that $\bfy_0^{\rm obs} = \bfy(0) + \epsilon$. In this case, we consider $\bfy_0 = [0.1,0.3]^{\top}$ and we contaminate it with Gaussian noise with $0$ mean and $0.05$ standard deviation. Again, we attempt to predict the data at $T=200$. We plot the results from 1000 simulations in \Cref{fig:ts} (right). Again, we see that the incomplete information transforms into a density estimation problem.
\end{example} 

The predator-prey model illustrates a fundamentally different scenario that is much more common in practice. \Cref{ex2} represents a much more realistic scenario, since in reality data is almost always sampled only partially or at low resolution and includes noise. For example, in weather prediction, where the primitive equations are integrated to approximate global atmospheric flow, it is common that we do not observe important variables such as pressure, temperature, or wind speed in sufficient resolution. In such settings, past observations are generally inadequate for accurate future prediction. While we could use classical ML techniques to solve closed systems, applying such techniques to open systems is often infeasible. This is because the system does not have a unique solution, given partial or noisy data. Motivated by the two examples above, we now form the learning problem of probabilistic forecasting (see, e.g. \cite{reich2015probabilistic}). 

\begin{definition}
    \label{def:stoch}{\bf Probabilistic Forecasting.}
    Probabilistic forecasting predicts future outcomes by providing a range of possible scenarios along with their associated probabilities, rather than a single point estimate. 
    
    Let the initial vector $\bfy(0) \sim \pi_0(\bfy)$ and assume that $\bfy(T)$ is obtained from integrating the dynamical system ${\cal D}$. Probabilistic forecasting refers to estimating and sampling from the distribution $\pi_T(\bfy)$.  
\end{definition}

This approach acknowledges that unique predictions may be unattainable and instead generates samples of future outcomes. By contrast, deterministic prediction relies on a closed system and highly accurate data, assumptions that often do not hold in practice. 

\section{Probabilistic Forecasting and Flow Matching}
\label{sec:probForc}

We use a Flow Matching (FM) variant to propagate uncertainty and generate the initial distribution. Since complete state observation is often impossible in these systems, deterministic predictions are infeasible. Consequently, stochastic prediction becomes the natural choice. Consider a system where observations are noisy or partial. In particular, let $\bfy(t)$ be a vector that is obtained from an unknown dynamical system ${\cal D}$ and let
\begin{eqnarray}
    \bfq(t) = \bfS \bfy(t) + \bfepsilon,
\end{eqnarray}
where $\bfS$ is a sampling operator that typically reduces the dimension of $\bfy$ and $\bfepsilon \sim N(0, \sigma^2\bfI)$.
The distinction between $\bfy$ and $\bfq$ is critical: $\bfy$ represents the full state, while $\bfq$ represents only partial knowledge.

Our goal is to predict samples from the distribution of observed partial states $\bfq_T = \bfq(t+T) \sim \pi_T(\bfq)$ given a sample from the distribution of the data at $\bfq_0 = \bfq(t) \sim \pi_0(\bfq)$.
Although explicit functional forms for $\pi_0$ and $\pi_T$ are unavailable, we can estimate their probability densities given sufficient samples from both distributions. This assumption is unrealistic without additional constraints, since we typically observe only a single time series.
We make the following standard assumptions: (i) The time-dependent process is autoregressive, allowing us to treat each time step as an independent sample $\bfq_0$ with corresponding future values $\bfq_T$. Although standard in the literature, this assumption may not hold in all practical applications. (ii) Periodicity and clustering. Many problems in Earth science exhibit natural periodicity, enabling data to be organized by annual cycles. 

Since we can obtain data samples from probability distributions at time $0$ and at time $T$, we can use the data to provide a stochastic prediction. Consequently, stochastic prediction reduces to a probability transformation problem rooted in mass transport \citep{BB2000}. While efficient solutions for low-dimensional problems have been addressed \citep{fisher1970statistical, villani2009optimal}, solutions for problems in high dimensions are only recently being developed. One recent and highly effective approach to solving such problems is stochastic interpolation (SI) \citep{albergo2022building}, which belongs to the broader family of FM methods. The basic idea of stochastic interpolation is to generate a simple interpolant for all possible points in $\bfq_0$ and $\bfq_T$. A simple interpolant of this kind is linear and reads
\begin{eqnarray}
    \bfq_t = t\bfq_T + (1-t)\bfq_0
\end{eqnarray}
where $\bfq_T \sim \pi_T$, $\bfq_0 \sim \pi_0$, and $t\in[0,1]$ is a parameter. The points $\bfq_t$ are associated from a distribution $\pi_t(\bfq)$ that converges to $\pi_T(\bfq)$ at $t=1$ and to $\pi_0(\bfq)$ at $t=0$.
In FM, one learns (i.e., estimates) the velocity 
\begin{eqnarray}
   \bfv_t(\bfq_t) = \dot\bfq_t =\bfq_T -\bfq_0 ~,
\end{eqnarray}
by solving the stochastic optimisation problem
\begin{eqnarray}
    \label{eq:vel_learn}
    \min_{\bftheta} {\frac 12}{\mathbb E}_{\bfq_0, \bfq_1}\int_0^1 \|\bfv_{\bftheta}(\bfq_t, t) - (\bfq_T -\bfq_0) \|^2 \, dt ~.
\end{eqnarray}
Here $\bfv_{\bftheta}(\bfq_t, t)$ is a function approximator for the velocity $\bfv_t$ that is given at points $\bfq_t$, and $\bftheta$ are the trainable parameters. A common function approximator used for $\bfv_{\bftheta}(\bfq_t, t)$ is a deep neural network.

During inference, we assume that the velocity model $\bfv_{\bftheta}(\bfq_t, t)$ has been trained and use it to evolve $\bfq$ from time $0$ to $T$. Specifically, we solve for $\bfq_T = \bfq(T)$ by
\begin{eqnarray}
    \label{eq:ode}
    {\frac {d\bfq}{dt}} = \bfv_{\bftheta}(\bfq, t) \quad \text{and}, \quad \bfq(0) = \bfq_0~.
\end{eqnarray}
We employ a deterministic rather than stochastic framework. This allows the incorporation of high-accuracy numerical integrators that can take larger step sizes. However, this implies that solving the ODE with a single initial condition $\bfq_0$ yields a single prediction. In order to sample from the target distribution $\pi_T$, we sample $\bfq_0$ from $\pi_0$ and then use the ODE (\Eqref{eq:ode}) to push many samples forward. We thus obtain many samples for $\bfq_0$ from $\pi_0$ (see next section) and use them to sample from $\pi_T$.
\begin{comment}
    Note that the ODE obtained for $\bfq$ is not physical. It is merely used to interpolate the density from time $0$ to $T$. To demonstrate, we continue with the predator-prey model.
\end{comment}

\section{Sampling the Perturbed States} \label{sec:GeneratePerturbation}

Sampling from the initial distribution $\pi_0$ is a robust strategy for ensuring physical consistency. Several approaches can achieve this goal. A brute-force approach searches the dataset for similar states \cite{delle2013probabilistic}. For example, given a reference state $\bfq_0$, we retrieve all states $\|\bfq_0 - \bfq_i\|^2 \le \epsilon$. This approach is feasible in low dimensions but intractable in high dimensions. To address this, we employ a generative ML approach to produce realistic perturbations for a single input state.

To address this limitation, we employ a flow-based Variational Autoencoder (VAE) \citep{dai2019diagnosing, vahdat2021score}. In the encoding stage, the VAE maps data $\bfq_0$ from the original distribution $\pi_0(\bfq)$ to a latent vector $\bfz$ drawn from a standard Gaussian $N(0,\bfI)$. In the decoding stage, the model maps the latent vector $\bfz$ back to $\bfq_0$. Since the Gaussian space is convex, small perturbations to the latent vector $\bfz$ generate samples from $\pi_0(\bfq)$ that are perturbed around the original $\bfq_0$.

The difference between standard VAEs and flow-based VAEs is that, given $\bfq_0 \sim \pi_0$, standard VAEs attempt to learn a transformation directly to a Gaussian space $N(0,\bfI)$. However, as demonstrated in \cite{ruthotto2021introduction}, traditional VAEs produce latent representations that are only approximately Gaussian. This deviation from a true Gaussian distribution makes perturbation and sampling more difficult.
Flow-based VAEs learn Gaussian mappings with significantly higher accuracy. Furthermore, such flows can learn the decoding from a Gaussian distribution back to the original distribution, as the learned transformation is invertible (see Appendix \ref{app:proof} for proof) and continuous time normalizing flow \cite{chen2018neural, grathwohl2018ffjord, albergo2022building}.  

\begin{figure}[t]
    \centering
\includegraphics[width=1.0\linewidth]{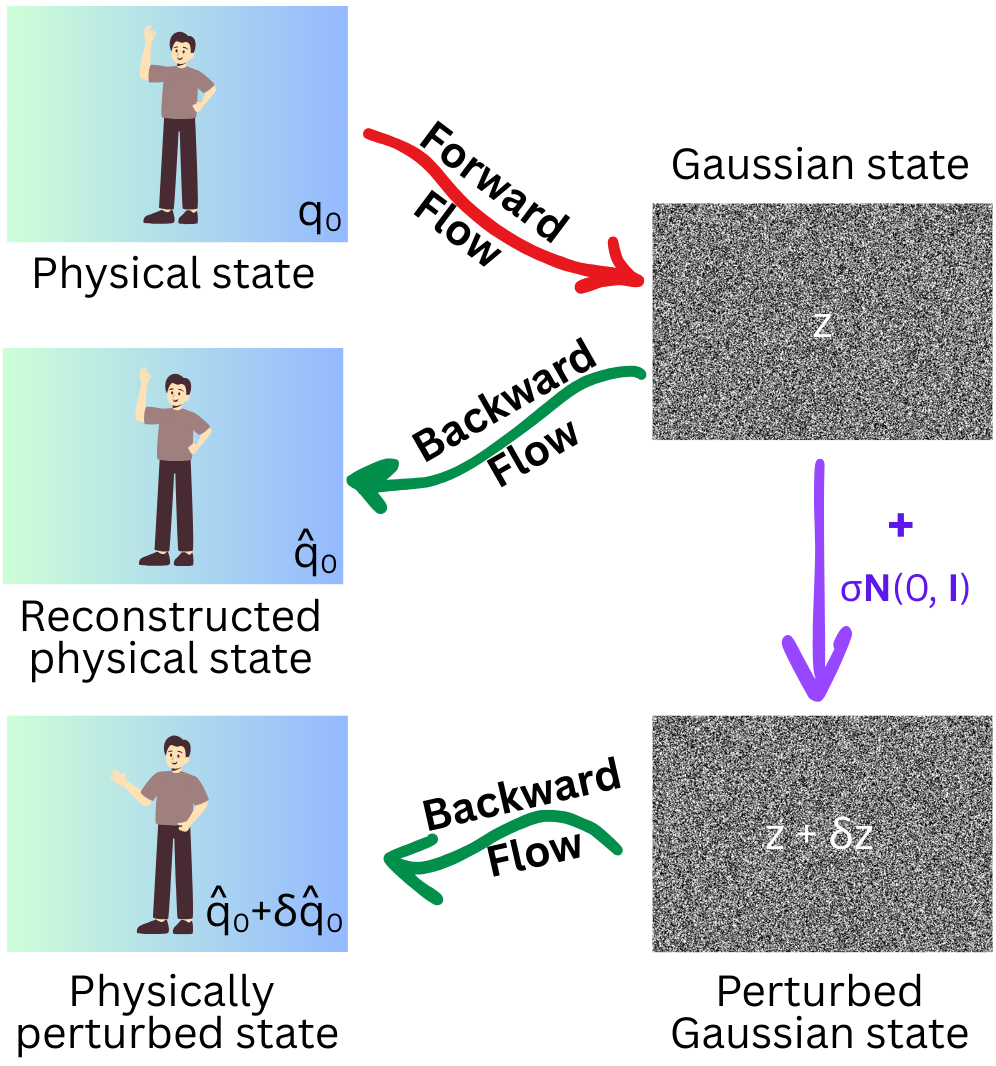}
    \caption{Diagram to show the process for obtaining a physical perturbed state.}
    \label{fig:physical_perturb}
\end{figure}

Such flows constitute a special case of FM, a deterministic scheme for learning a transformation from a physical state $\bfq_0$ to a latent Gaussian state $\bfz$. Since the transformation is invertible, we can easily reverse-integrate from the latent space back to the physical space.

We define the interpolant  as
\begin{equation}
    \bfq_t = (1 - t)\bfq_0 + t\bfz~, \text{ if } t \in [0, 1].
\end{equation}

The velocity field associated with the above interpolant is
\begin{equation}
\bfu_t = {\frac {d\bfq_t}{dt}} = \bfz - \bfq_0 ~, \text{ if } t \in [0, 1].
\end{equation}

Note that the flow starts at $\bfq_0 \sim \pi_0$ towards $\bfz \sim N(0,\bfI)$ at $t=0$ and arrives to $\bfz$ at $t=1$. This is the encoding state. Training these models is straightforward and follows a similar procedure to training our flow matching model, namely, we solve a stochastic optimisation problem of the form
\begin{eqnarray}
    \label{eq:vel_learnz}
    \min_{\bftheta}{\mathbb E}_{\bfq_0, \bfz}\int_0^{1} \|\bfu_{\bftheta}(\bfq_t, t) - (\bfz - \bfq_0) \|^2 dt.
\end{eqnarray}

Given the velocity $\bfu_{\bftheta}$, we now generate perturbed samples in the following way. First, we integrate the ODE from $t=0$ to $t=1$, that is, we solve the ODE
\begin{eqnarray}
    \label{eq:ode1}
    {\frac {d\bfq}{dt}} = \bfu_{\bftheta}(\bfq, t), \quad \text{given } \quad \bfq(0) = \bfq_0.
\end{eqnarray}
This yields the latent state $\bfq(1) = \bfz \sim N(0, \bfI)$.
We then perturb the encoded latent state according to
\begin{eqnarray}
    \label{eq:pert}
    \widehat \bfq(1) = \bfq(1) + \sigma \boldsymbol{\omega},
\end{eqnarray}
where $\boldsymbol{\omega} \sim N(0,\bfI)$ and $\sigma$ is a hyper-parameter. 

Note that we can use the same velocity field for the reverse decoding process. The decoding map pushes the points $\widehat \bfq(1)$ back to $\widehat \bfq_0$ by solving for $\widehat \bfq_{0} = \bfq(t=0)$ such that 
\begin{eqnarray}
    \label{eq:ode2}
    {\frac {d\bfq}{dt}} = -\bfu_{\bftheta}(\bfq, t) \quad \text{and}, \quad \bfq(1) = \widehat \bfq(1)~.
\end{eqnarray} 

We then integrate the ODE \Eqref{eq:ode2} from $1$ to $0$ starting from $\widehat \bfq(1)$, obtaining a perturbed state $\widehat \bfq_0$ for an initial state $\bfq_0$. The magnitude of the perturbation is controlled by the hyperparameter $\sigma$. We decode $M$ perturbed latent states to obtain $M$ perturbed samples, which define the input distribution $\pi_0$ for ensemble prediction as described in \Cref{sec:probForc}.\\

\textbf{\emph{Theorems, proofs, and pseudocode are provided in the Appendix \ref{app:proof}.}}

\section{Experiments}\label{sec:expt}
We apply the proposed framework to diverse challenging temporal datasets to demonstrate its effectiveness for probabilistic forecasting and uncertainty quantification. We conducted experiments on two synthetic datasets that allow for full analytical evaluation, as well as on a real-world weather dataset. 

\subsection{Datasets}\label{subsec:datasets}
The datasets and experimental setup are described below.

\textbf{Lotka–Volterra Predator–Prey model:} A nonlinear dynamical system governed by \Eqref{eq:pp}, with parameters $\bfp_1$ = 2/3, $\bfp_2$ = 4/3, $\bfp_3$ = 1 and $\bfp_4$ = 1. The time series model takes the current state as input to predict the state after a time step of $t = 200$. This long integration time produces a widely distributed and complex output probability distribution.
\textbf{Architecture:} multilayer perceptron.
\textbf{Test case:} Standard test set for quality of forecasting. Uncertainty is introduced by adding Gaussian noise with mean $\bfy(0) = [0.1, 0.3]^T$ and standard deviation $\sigma = 0.05$. The distribution of final states is obtained by numerically integrating over $t \in [0, 200]$ starting from noisy initial states. 

\textbf{MovingMNIST:} A variant of the MNIST dataset \citep{MovingMNIST}. This is a synthetic video dataset designed to benchmark video prediction models. Each sequence features 20 frames of two MNIST digits moving along randomly chosen trajectories with varying speeds and directions. Following the setup described in \cite{MovingMNIST}, a predictive model takes the first 10 frames as input to predict the subsequent 10 frames.
\textbf{Learning functional:} U-Net with attention \cite{dhariwal2021diffusion}.
\textbf{Test case:} Standard test set for quality of forecasting. Uncertainty is introduced by adding Gaussian noise to the initial velocities and direction angles during video sequence generation. 

\textbf{WeatherBench:} We use a subset of the WeatherBench dataset \citep{weatherBench}, which is derived from ERA5 reanalysis data \citep{era5hersbach} and includes three downsampled variants. Specifically, we use the version downsampled to a 5.625$^{\circ}$ resolution on a latitude–longitude grid. Details of variables and model configurations are provided in Appendix \ref{app:WBvars}. We follow the experimental setup described by \citet{nguyen2023climax}. We use a 6-hour reanalysis window as the model time step.
\textbf{Learning functional:} U-Net with attention \cite{dhariwal2021diffusion}.
\textbf{Test case:} Standard test set for quality of forecasting. To evaluate uncertainty quantification, we select a test state and generate the initial distribution by identifying similar states in the full dataset. 

The statistics and configurations for each dataset used in the experiments are provided in Appendix~\ref{app:datasetstats}. Additional experiments on Vancouver93 (a 93-channelled timeseries of temperature) and CloudCast (a satellite spatiotemporal dataset) are provided in the Appendix. In the next section, we provide details of the experiments and results for the predator-prey model, MovingMNIST, and WeatherBench datasets.

\textbf{\emph{See Appendix for more experiments on additional datasets.}}

\subsection{Experiment and Results}

We evaluate the framework using test cases with uncertainty introduced via classical perturbation methods. For each dataset, we thus have access to ground-truth distributions for both initial and final states, enabling statistical evaluation of the predicted ensembles. We compare against two state-of-the-art methods: Denoising Diffusion Probabilistic Models (DDPM) \cite{ho2020denoising} and Probabilistic Flow Matching with Föllmer's process (PFI) \cite{chen2024probabilistic}. Both are SDE-based conditional diffusion models that underpin most popular probabilistic forecasting approaches (see \citet{price2025probabilistic, li2024generative, bonev2025fourcastnet}).  

To model initial state uncertainty, we apply the generative flow matching method from \Cref{sec:GeneratePerturbation} to create an ensemble of initial conditions. We then propagate these forward using deterministic flow matching to generate an ensemble of final states. We compare the predicted distributions against ground-truth values and standard ensemble forecasting baselines using the metrics detailed below and in the Appendix.

\paragraph{Metrics:} \label{Result:genMetrics}
We evaluate the performance of our method using a set of statistical and image-based metrics. For each predicted ensemble, a corresponding ground-truth ensemble is available. The objective is to ensure that the statistical properties of the predicted ensemble closely match those of the target distribution.

Following recent work on probabilistic forecasting \cite{price2025probabilistic, oskarsson2024probabilistic}, we employ the CRPS between real target and our ensemble of predictions. We also used additional metrics for closer look at the ensemble. We also use mean and standard deviation scores, which measure the average and variability of pixel values across the state space. Ideally, these scores should match closely between the predicted and target ensembles, indicating similar distributions. Additionally, drawing from techniques in turbulent flow analysis \cite{pope2001turbulent}, we compare ensemble mean and standard deviation fields. For similarity assessment, we employ standard image metrics: Mean Squared Error (MSE), Mean Absolute Error (MAE), and Structural Similarity Index (SSIM). Lower MSE and MAE values indicate better alignment with the ground truth, while higher SSIM reflects better structural similarity and better alignment with the ground truth. Full metric definitions are provided in Appendix \ref{app:metrics}.

\begin{table*}[!ht]
\centering
\begin{tabular}{c|l|c|c|c|ccc|ccc}
\hline
& & & &  & & Mean & & & SD \\
& & & {Mean}& {Std Dev} & & State & & & State \\
Data & {Methods}      & CRPS ($\downarrow$) & {Score}     & {Score} & {MSE($\downarrow$)}& {MAE($\downarrow$)} & {SSIM($\uparrow$)} & {MSE($\downarrow$)}& {MAE($\downarrow$)} & {SSIM($\uparrow$)} \\\hline
 & VAR             & 5.56 &  2.1    &   3.7  &  1.1    &  8.1  &      &    3.1  &  5.7    &      \\
& DDPM(LSTM)     & 1.42e-1 & 7.54e-1     &  1.21    &  2.3e-2    &    1.9e-1   &      &    4.2e-2  &  2.9e-1    &      \\
Predator & \textbf{FM (Our)}          & \third{7.48e-2} & 7.43e-1     &  1.07    &  \third{8.0e-3}    &    \third{8.8e-2}   &      &    \third{2.7e-2}  &  \third{1.6e-1}    &      \\
Prey & PFI          & \second{6.88e-2} & 7.46e-1   &   1.06   &   \second{7.6e-3}   &    \second{7.3e-2}   &      &  \second{2.4e-2}    &  \second{1.1e-1}   &      \\
& \textbf{FMwS (Our)}   & \first{3.85e-2}  & 7.49e-1     &  1.06    &  \first{6.7e-3}    &    \first{5.2e-2}   &      &    \first{8.7e-3}  &  \first{8.1e-2}    &      \\
& Ground truth   & 0 & 7.55e-1     &  1.04    &  0    &    0  &      &    0  &  0    &          \\
\hline
 & DDPM(U-Net)    & 5.39e-2 &  5.42e-2     &  2.91e-1    &  8.9e-3    &    7.1e-2   &  0.723    &    9.2e-3  &  7.6e-2    &    0.622  \\
& DDPM(adrnet)     & 3.65e-2  &  5.37e-2     &  3.02e-1    &  6.4e-3    &    4.9e-2   &  0.791    &    7.4e-3  &  6.2e-2    &    0.671  \\
Moving & \textbf{FM (Our)   }    & \third{1.46e-2} & 5.58e-2     &  1.87e-1    &  \third{2.5e-3}    &    \third{1.8e-2}   &  \third{0.843}    &    \third{3.6e-3}  &  \third{2.5e-2}    &    \third{0.744}  \\
MNIST & PFI          & \second{1.03e-2} &   5.71e-2   &   2.93e-1   &   \second{1.9e-3}   &   \second{1.6e-2}    &  \second{0.879}    &   \second{2.7e-3}   &  \second{1.4e-2}   &   \second{0.803}   \\
 & \textbf{FMwS (Our)         }&  \first{5.67e-3} &  5.98e-2    &  2.34-1    &  \first{8.6e-4}    &    \first{8.6e-3}   &  \first{0.917}    &    \first{1.2e-3}  &  \first{9.1e-3}    &    \first{0.894}  \\
& Ground truth   & 0 & 6.01e-2     &  2.21e-1    &  0    &    0  &  1    &    0  &  0    &    1      \\
\hline
 & DDPM(U-Net)            & 3.12e2 &  1.31e4    &    9.87e1  &  6.1e4    &    4.7e2 &  0.585    &    4.3e4  &  1.1e2    &    0.487   \\
& DDPM(adrnet)         & 1.47e2 &  1.33e4    &  8.22e1    &  5.7e4    &   2.1e2   &  0.616    &    4.5e4  &  9.2e1    &    0.531  \\
Weather & \textbf{FM (Our)}       & \third{3.53e1} & 1.36e4     &  9.66e1    &  \third{1.9e4}    &    \third{5.2e1}   &  \third{0.868}    &    \third{1.1e4}  &  \third{3.6e1}    &    \third{0.608}  \\
Bench & PFI          & \second{2.40e1} &   1.36e4   &   9.45e1   &   \second{9.2e4}   &    \second{3.7e1}   &    \second{0.892}  &   \second{8.3e3}   &   \second{2.4e1}  &    \second{0.672}  \\
 & \textbf{FMwS (our)}         & \first{1.34e1} & 1.36e4     &  9.32e1    &  \first{6.3e3}    &    \first{2.2e1}   &  \first{0.911}    &    \first{1.6e3}  &  \first{1.0e1}    &    \first{0.805}  \\
& Ground truth   & 0 & 1.36e4     &  9.15e1    &  0    &    0  &  1    &    0  &  0    &    1      \\
\hline
\end{tabular}
\caption{Comparison of metrics for ensemble predictions. Best results are highlighted in \first{red}; second-best are in \second{violet}; third-best are in \third{black}.}
\label{tab:TotalstatsMetrics}
\end{table*}

\begin{figure}[!t]
    \centering
\includegraphics[width=1.0\linewidth]{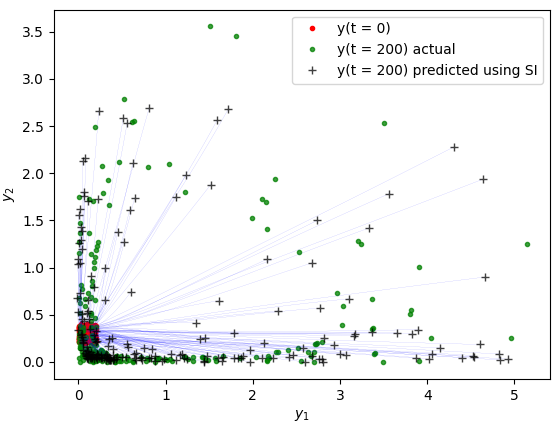}
    \caption{Comparison of the actual final distribution and that obtained using FM on the predator-prey model. Trajectories for transport learned by FM are in blue. Note that the trajectories are not physical.}
    \label{fig:pred-prey200}
\end{figure} 

\begin{figure*}[!t]
    \centering
    \includegraphics[width=1.0\linewidth]{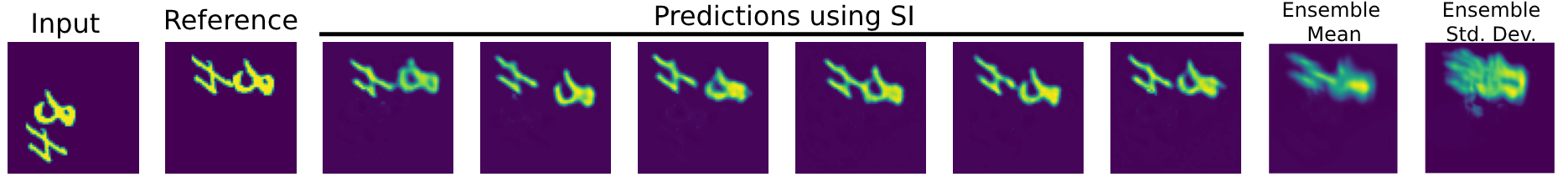}
    \caption{Six of 50 Moving MNIST trajectory predictions obtained using FM and their ensemble mean and standard deviation.}
    \label{fig:MMMultiPred}
\end{figure*}

\begin{figure*}[!t]
    \centering
    \includegraphics[width=1.0\linewidth]{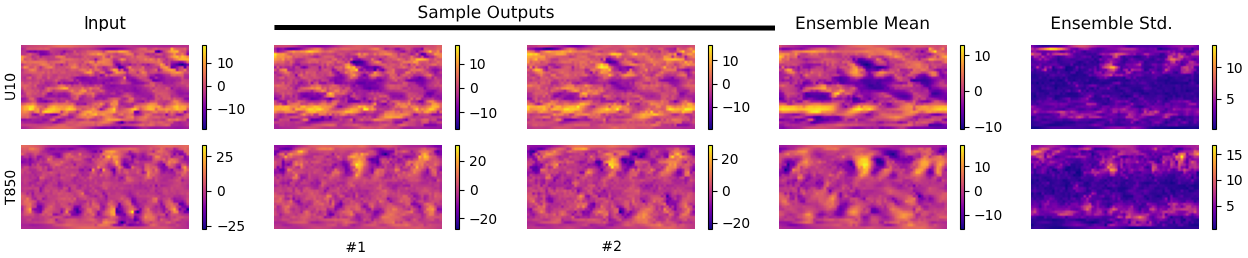}
    \caption{Sample stochastic forecasts and the ensemble statistical states of U10 and T850 for 100 forecasts for a 2-day lead time stochastic prediction using FM.}
    \label{fig:wBsampleforecasts}
\end{figure*}

{\bf Predator-Prey Model:}
\Cref{fig:pred-prey200} presents the initial and final ground-truth distributions alongside FM predictions, highlighting key findings. Our method closely matches the final distribution and accurately captures high-probability regions. The widely used Vector Autoregression (VAR) method \cite{stock2001vector}, commonly applied in high-dimensional probabilistic forecasting, fails to provide satisfactory predictions \cref{tab:TotalstatsMetrics}. This failure is due to the strong nonlinearity inherent in this system. It can be seen that DDPM and PFI using LSTM \cite{hochreiter1997long}, and our FM approach, are effectively various conditional generative models. However, our FMwS approach, which combines generative sampling of physically perturbed initial states with flow matching, achieves the best performance. Additional statistical comparisons, including probability estimates for state variables, are provided in Appendix~\ref{app:predpreystats}.\\

{\bf Moving MNIST and WeatherBench:}
As shown in Table \Cref{tab:TotalstatsMetrics}, DDPM with a U-Net \cite{dhariwal2021diffusion}, AdrNet \cite{zakariaei2024advection}, PFI with a U-Net, and our FM approach all effectively model uncertainty via perturbation-based sampling. However, our FMwS variant outperforms these baselines. 

The predictions for MovingMNIST are shown in \Cref{fig:MMMultiPred}. The predictions exhibit variability while remaining consistent with the ground-truth distribution. However, the ensemble mean and standard deviation closely match the deterministic unperturbed reference state. Additional statistical comparisons of the ensemble predictions using FM with the ensemble of actual targets are provided in Appendix~\ref{app:MMstats}. We similarly apply our approach to the WeatherBench dataset to develop a probabilistic forecasting model. As shown in \Cref{fig:wBsampleforecasts}, the model produces diverse stochastic outputs with state-of-the-art CRPS scores. The ensemble mean and standard deviation closely match the reference values. We present the standard variables U10 and T850 here for clarity. Additional statistical comparisons for Z500, T2m, U10, and T850 are provided in Appendix~\ref{app:wBstats}. 

\section{Discussions}\label{sec:discussions}
The experimental results show that our deterministic approach is competitive with state-of-the-art conditional diffusion models while enabling uncertainty quantification. However, our approach is simpler, more efficient, and maintains physical interpretability. See Appendix \ref{app:ODEvsSDE} for a detailed computational comparison with typical diffusion models. Theoretically, our method requires two neural network evaluations per ensemble member, compared to more than 50 evaluations for the diffusion-based methods. Moreover, Appendix~\ref{app:MM} presents ablation studies examining the sensitivity to perturbation types and noise levels. These studies validate our proposed approach. We now discuss several limitations and avenues for future work. Our method is best suited for problems with small time steps and low intrinsic stochasticity in the source time series. However, our approach can quantify uncertainty in deterministic systems, which is non-trivial for purely stochastic models. This occurs because stochastic models applied to deterministic time series cannot effectively capture the true uncertainty structure, learning instead to ignore stochasticity. However, our method may struggle with inherently stochastic systems (e.g., financial time series) or long-range prediction problems. This suggests hybrid approaches combining our method with diffusion-based techniques for improved performance.      

\section{Conclusion}\label{sec:conclusion}

Probabilistic forecasting is a fundamental research area across numerous scientific disciplines. While diffusion-based generative models have improved uncertainty modeling, they require significant computational resources. In this work, we combine flow matching with normalizing flows to generate physically meaningful perturbations, enabling explicit uncertainty quantification in dynamical systems. To our knowledge, this application of generative models for perturbation generation has not been previously explored.

We then employ a second flow matching model, a deterministic distribution-to-distribution mapping, to propagate uncertainty over extended time horizons. Our approach is based on ordinary differential equations (ODEs) rather than stochastic differential equations (SDEs), enabling faster training and inference. Furthermore, we decouple stochasticity, introduced at inference time, from the deterministic training process. This design eliminates common constraints while enabling greater flexibility in application.

Our experiments across diverse benchmarks validate the effectiveness of this approach. We achieve state-of-the-art probabilistic scoring while maintaining computational efficiency and physical interpretability. The framework's modularity is a key advantage: both components can be used independently or combined with other machine learning approaches. This flexibility provides practitioners with powerful tools for uncertainty quantification in complex dynamical systems.





\newpage

\bibliography{uai2026-template}

\newpage

\onecolumn

\title{Fast and Flexible Probabilistic Forecasting of Dynamical Systems\\ using Flow Matching and  Physical Perturbation (Supplementary Material)}
\maketitle

\appendix

\section{Definitions, Theorems and Proofs} \label{app:proof}

Given that there is a dynamical system ${\cal D}$ of the form
\begin{eqnarray}
    \dot \bfy = f(\bfy, t, \bfp),
\end{eqnarray}
where, $\bfy \in \mathbb{R}^n$ is a n-dimensional state vector and $f$ is a function that depends on the state $\bfy$, the time $t$ and a k-dimensional parameter $\bfp \in \mathbb{R}^k$.

\begin{definition}\textbf{Physical Space} \label{def:PhysicalSpace}
The physical space ${\cal Y}$ is defined as the vector space of observed data $\bfy \in \mathbb{R}^n$ from an n-dimensional dynamical system ${\cal D}$.
\end{definition}

\begin{definition}\textbf{Gaussian Space} \label{def:GaussSpace}
An n-dimensional Gaussian space ${\cal G}$ can be defined as the vector space of $\bfz \in \mathbb{R}^n$ such that $\bfz \sim \mathcal{N}(0, I)$ and $I$ is the n-dimensional identity matrix.
\end{definition}

\begin{definition}\textbf{Physical Perturbation} \label{def:PhysicalPerturbation}
Let $\bfy$ be an n-dimensional vector in physical space $\cal Y$ for the dynamical system $\cal D$. Then, physical perturbation is a function $ \cal P_D : \cal Y \rightarrow \cal Y$ defined as 
\begin{eqnarray}
    \cal P_D(\bfy) = \bfy + \delta_{\cal D},
\end{eqnarray}
such that given a perturbation vector $\delta_{\cal D} \in \mathbb{R}^n$, and for a constant $\epsilon \in \mathbb{R}$,
\begin{eqnarray}
     \parallel \cal{P_D}(\bfy) - \bfy  \parallel^\text{2} ~\le ~\epsilon .
\end{eqnarray}

\end{definition}

\begin{definition}\textbf{Gaussian Perturbation} \label{def:GaussPerturbation}
A Gaussian perturbation is a function $ \cal P_G : \mathbb{R}^{\text{n}} \rightarrow \mathbb{R}^{\text{n}}$ over any n-dimensional data $\bfy \in \mathbb{R}^n$ defined as 
\begin{eqnarray}
    \cal P_G(\bfy) = \bfy + \delta_{\cal G},
\end{eqnarray}
such that given a perturbation vector $\delta_{\cal G} \sim \mathcal{N}(0, \sigma^2 I)$, where $\sigma$ is the standard deviation of the perturbation and \( I \) is the identity matrix, and for a constant $\epsilon \in \mathbb{R}$
\begin{eqnarray}
     \parallel \cal{P_G}(\bfy) - \bfy  \parallel^\text{2} ~\le ~\epsilon .
\end{eqnarray}
\end{definition}

\begin{assumption}
Let ${\cal Y}$ be the space of observed data $\bfy \in \mathbb{R}^n$ from a dynamical system ${\cal D}$ and Gaussifying flow matching function learns a linear interpolation between $\bfy$ and $\bfz \sim \mathcal{N}(0, I)$, where $I$ is the n-dimensional identity matrix. Then the interpolant $\bfy_t$ is defined as 
\begin{eqnarray}
     \bfy_t(\bfy, \bfz, t) = (1-t)\bfy + t\bfz, 
\end{eqnarray}
where $t \sim U(0,1)$.
\end{assumption}

\begin{definition}\textbf{Encoder} \label{def:Genc}
Let $\bfy$ be an n-dimensional vector in physical space $\cal Y$ for the dynamical system $\cal D$ and $\bfz$ be an n-dimensional vector in the latent Gaussian space $\cal{G}$. Encoder is a function $ \cal{E} : \cal{Y} \rightarrow \cal{G}$ defined as 
\begin{eqnarray}
    \cal{E}(\bfy) = \bfy + \int_{\text{0}}^{\text{1}}\text{v}(\bfq(\text{t}), \theta, \text{t) dt},
\end{eqnarray}
where velocity $\text{v}$ is a function of $\bfq$(t), a scalar t $\in \mathbb{R}$ and trained parameter $\theta$ such that 
\begin{eqnarray}
    \bfq(\text{t}) = \bfy + \int_{\text{0}}^{\text{t}}\text{v}(\bfq(\text{t}), \theta, \text{t) dt}.
\end{eqnarray}

If we assume that the flow matching velocity $\text{v}$ is learned accurately, then 
\begin{eqnarray}
  \text{v}(\bfq(\text{t}), \theta, \text{t}) = \bfz - \bfy,
\end{eqnarray}
and
\begin{eqnarray}
  \cal{E}(\bfy) = \bfz.
\end{eqnarray}
\end{definition}

\begin{definition}\textbf{Decoder} \label{def:Gdec}
Let $\bfy$ be an n-dimensional vector in physical space $\cal Y$ for the dynamical system $\cal D$ and $\bfz$ be an n-dimensional vector in the latent Gaussian space $\cal{G}$. Decoder is a function $ D : \cal{G} \rightarrow \cal{Y}$ defined as 
\begin{eqnarray}
    D(\bfz) = \bfz - \int_{\text{0}}^{\text{1}}\text{v}(\bfq(\text{t}), \theta, \text{t) dt},
\end{eqnarray}

where velocity $\text{v}$ is a function of $\bfq$(t), a scalar t $\in \mathbb{R}$ and trained parameter $\theta$ such that 
\begin{eqnarray}
    \bfq(\text{t}) = \bfy - \int_{\text{t}}^{\text{1}}\text{v}(\bfq(\text{t}), \theta, \text{t) dt}.
\end{eqnarray}

If we assume that the flow matching velocity $\text{v}$ is learned accurately, then 
\begin{eqnarray}
  \text{v}(\bfq(\text{t}), \theta, \text{t}) = \bfz - \bfy,
\end{eqnarray}
and
\begin{eqnarray}
  D(\bfz) = \bfy.
\end{eqnarray}
\end{definition}

\begin{lemma}[Inverse Property]\label{lemma1}
Let ${\cal Y}$ be the physical space for an n-dimensional dynamical system ${\cal D}$. Given an accurately trained Gaussifying flow matching function on the data in $\cal Y$, the encoder $\cal E$ and decoder $D$ are inverses of each other.
\end{lemma}

\begin{proof}
    Assume the Gaussifying flow matching function is trained to convergence. Then for data $\bfy \in \cal Y$
    \begin{eqnarray}
      \cal{E}(\bfy) = \bfz,
    \end{eqnarray}
    where $\bfz \sim \mathcal{N}(0, I)$ and $I$ is the n-dimensional identity matrix, see \Cref{def:Genc}.\\

    If we decode $\cal{E}(\bfy)$ (see \Cref{def:Gdec}), we get 
    \begin{eqnarray}
      \text{D}(\cal{E} (\bfy)) = \text{D}(\bfz).\\
      \implies \text{D}(\cal{E}(\bfy)) = \bfy. 
    \end{eqnarray}
    Therefore, $D$ is the left inverse of $\cal E$.\\
    
    Also similarly,
    \begin{eqnarray}
      \text{D}(\bfz) = \bfy.
    \end{eqnarray}
    By operating with the encoder, we get
    \begin{eqnarray}
      \cal{E} (\text{D}(\bfz)) = \cal{E}(\bfy).\\
      \implies \cal{E} (\text{D}(\bfz)) = \bfz.
    \end{eqnarray}
    Consequently, D is the right inverse of $\cal E$.

    Since D is both left inverse and right inverse of $\cal E$,
    \begin{eqnarray}
      \text{D} = \cal E^{-\text{1}}.
    \end{eqnarray}
    
\end{proof}

\begin{theorem}[Physical Perturbation]\label{theorem1}
Let ${\cal Y}$ be the space of observed data $\bfy \in \mathbb{R}^n$ from an n-dimensional dynamical system ${\cal D}$. Given an accurately trained Gaussifying flow matching function on the data in $\cal Y$, a Gaussian perturbation $\cal P_G$ of the encoded latent vector in an encoder-decoder setup results in a physical perturbation $\cal P_D$ of data $\bfy \in \cal Y$.

Mathematically,
\begin{eqnarray}
    D(\cal{P_G}(\cal{E}(.))) \subseteq \cal{P_D}(.)
\end{eqnarray}
\end{theorem}

\begin{proof}
Let us propose a perturbation $\cal{P}$ over $\bfy$ as
\begin{eqnarray}
    \cal{P}(\bfy) = D(\cal{P_G}(\cal{E}(\bfy)))
\end{eqnarray}

We know that,
\begin{eqnarray}
    \cal{E}(\bfy) = \bfz.\\
    \implies D(\cal{P_G}(\cal{E}(\bfy))) = D(\cal{P_G}(\bfz)).
\end{eqnarray}

Also, we know that by \Cref{def:GaussPerturbation}
\begin{eqnarray}
    \cal{P_G}(\bfz) = \bfz + \delta_G.
\end{eqnarray}

Let $\cal G$ be an n-dimensional Gaussian space. We see that,
\begin{eqnarray}
    \bfz \in \cal{G} \text{ and } \delta_G \in \cal{G}.
\end{eqnarray}

It is a standard result that the sum of two Gaussian noise vectors is also a Gaussian noise vector \citet{lemons2002introduction}. Therefore,
\begin{eqnarray}
    \cal{P_G}(\bfz) \in \cal{G}.
\end{eqnarray}

We know that $D:\cal{G} \rightarrow \cal{Y}$, from which it follows that 
\begin{eqnarray}
    \cal{P_G}(\bfz) \in \cal{G},\\
    \implies D(\cal{P_G}(\bfz)) \in \cal{Y}.\\
    \implies D(\cal{P_G}(\cal{E}(\bfy))) \in \cal{Y}.\\
    \implies \cal{P}(\bfy) \in \cal{Y}.\\
    \implies \cal{P}(.) \subseteq \cal{P_D}(.).
\end{eqnarray}

Hence, the functional space of proposed perturbations is a valid functional subspace of physical perturbations.

\end{proof}

\newpage
\section{Algorithms} \label{app:pseudocode}
The algorithms underlying our methods are presented below. See \Cref{alg:FMDS} for training the predictive model, \Cref{alg:gaussification} for training the perturbing model, and \Cref{alg:perturb} for generating perturbations using the trained perturbing model. 

\begin{algorithm}[b!]
\caption{FM for training dynamical systems}
\label{alg:FMDS}
\begin{algorithmic}[1]
\REQUIRE
Dataset $\{x_i\}_{i=1}^N$,

lead time $T$ for prediction,

time sampler $p_t$, and

function approximator $\hat{u}_{\theta^{'}}(x, t)$.\newline

\WHILE {not converged}
    \STATE sample minibatch $\{x_i\}_{i=1}^B$ from dataset\;
    \STATE obtain minibatch $\{x_{i+T}\}_{i=1}^B$ from dataset\;
    \STATE sample $t \sim p_t$ (e.g., uniform over $[0, 1]$)\;

    \STATE interpolate points:\\ $x_t^i \gets (1 - t) x_i + t x_{i+T}$ \;
    
    \STATE compute target velocity:\\ $u_t^i \gets x_{i+T} - x_i $ \;

    \STATE predict:\\ $\hat{u}_t^i \gets \hat{u}_{\theta^{'}}(x_t^i, t)$\;

    \STATE compute loss:\\ $\mathcal{L^{'}} \gets \frac{1}{B} \sum_{i=1}^B \|\hat{u}_t^i - u_t^i\|^2$\;

    \STATE update $\theta^{'}$ with gradient descent:\\ ${\theta^{'}}$ $\gets$ ${\theta^{'}}$ - $\eta$ $\nabla_{\theta^{'}} \mathcal{L^{'}}$\;
\ENDWHILE \newline
\RETURN{trained timeseries prediction flow field $\hat{u}_{\theta^{'}}(., .)$ in functional form}
\end{algorithmic}
\end{algorithm}

\begin{algorithm}[b!]
\caption{FM for training a Gaussifying transformation}
\label{alg:gaussification}
\begin{algorithmic}[1]
\REQUIRE
Dataset $\{x_i\}_{i=1}^N$,

base distribution $p_0$ (e.g. $\mathcal{N}(0, I)$),

time sampler $p_t$, and

function approximator $\hat{u}_{\theta^{'}}(x, t)$.\newline

\WHILE {not converged}
    \STATE sample minibatch $\{x_i\}_{i=1}^B$ from dataset\;
    \STATE sample $z_i \sim p_0$ (standard normal)\;
    \STATE sample $t \sim p_t$ (e.g., uniform over $[0, 1]$)\;

    \STATE interpolate points:\\ $x_t^i \gets (1 - t) x_i + t z_i$ \;
    
    \STATE compute target velocity:\\ $u_t^i \gets z_i - x_i $ \;

    \STATE predict:\\ $\hat{u}_t^i \gets \hat{u}_{\theta^{'}}(x_t^i, t)$\;

    \STATE compute loss:\\ $\mathcal{L^{'}} \gets \frac{1}{B} \sum_{i=1}^B \|\hat{u}_t^i - u_t^i\|^2$\;

    \STATE update $\theta^{'}$ with gradient descent:\\ ${\theta^{'}}$ $\gets$ ${\theta^{'}}$ - $\eta$ $\nabla_{\theta^{'}} \mathcal{L^{'}}$\;
\ENDWHILE \newline
\RETURN{trained Gaussifying flow field $\hat{u}_{\theta^{'}}(., .)$ in functional form}
\end{algorithmic}
\end{algorithm}

\begin{algorithm}[h!]
\caption{Generating an ensemble of perturbed states using Gaussifying velocity function and Euler integration}
\label{alg:perturb}
\begin{algorithmic}[1]

\REQUIRE
Initial state $x_0$,

trained Gaussifying flow field function $\hat{u}_{\theta^{'}}(x, t)$,

prediction horizon $T$,

number of steps $N$,

perturbation level $\sigma$, and

number of predictions $M$.\newline

\STATE initialize $x \gets x_0$\;
\STATE initialize $t \gets 0$\;
\STATE set time step $\Delta t \gets T / N$\;
\FOR{$i = 1$ \TO $N$}

    \STATE $u \gets \hat{u}_\theta^{'}(x, t)$\;

    \STATE $x \gets x + \Delta t \cdot u$\;  

    \STATE $t \gets t + \Delta t$\;

\ENDFOR \newline


\ENSURE $t = 1$
\STATE initialize an empty set: $X$\;        
\FOR{$i = 1$ \TO $M$}
    \STATE initialized perturb state:\\ $x'\gets x + \sigma \mathcal{N}(0, I) $\;
    \STATE initialize $t \gets 1$\;
    \FOR{$j = 1$ \TO $N$}
    
        \STATE $u \gets \hat{u}_\theta^{'}(x', t)$\;

        \STATE $x' \gets x' - \Delta t \cdot u$\;

        \STATE $t \gets t -  \Delta t$\;

    \ENDFOR
    
    \STATE append $x'$ to $X$\;
\ENDFOR \newline

\RETURN{predicted set of perturbed states $X$}
\end{algorithmic}
\end{algorithm}

\clearpage
\section{Computational Cost: ODEs vs. SDEs} \label{app:ODEvsSDE} 

\begin{figure}[!h]
\begin{minipage}[t]{0.5\textwidth}
    \centering
    \includegraphics[width=\linewidth]{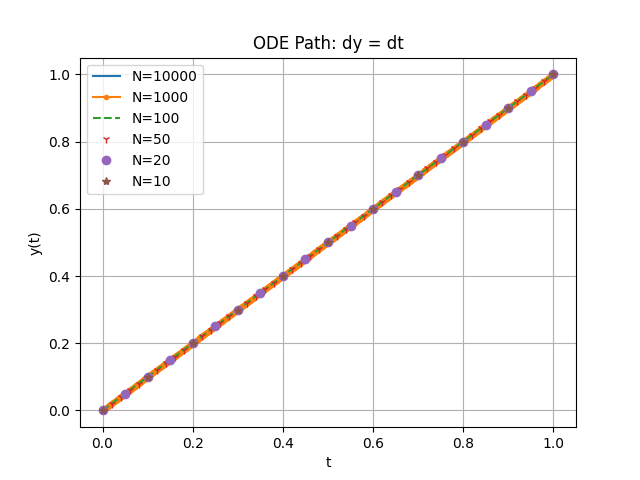}
    \caption{ODE solution for different step sizes.}
    \label{fig:ode} 
\end{minipage}
\hfill
\begin{minipage}[t]{0.5\textwidth}
    \centering
    \includegraphics[width=\linewidth]{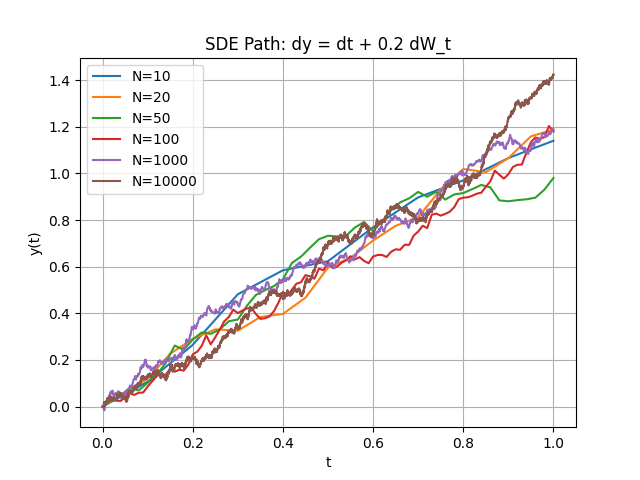}
    \caption{SDE solution for different step sizes.}
    \label{fig:sde}
\end{minipage}
\end{figure}

\subsection{ODE Problem}

We are given a trivial linear initial value problem like FM:
\[
\frac{dy}{dx} = 1, \quad y(0) = 0
\]
The objective is to approximate $y(1)$ using Euler's method.

\subsubsection{Euler's Method Formula}
Euler's method updates the solution using:
\[
y_{n+1} = y_n + h \cdot f(x_n, y_n)
\]
where \( f(x, y) = \frac{dy}{dx} = 1 \), and \( h \) is the step size.

\subsubsection{Step-by-Step Approximation}
Let the interval \([0, 1]\) be divided into \( N \) steps with step size:
\[
h = \frac{1}{N}
\]

Set \( y_0 = 0 \), and iterate:
\[
\text{For } n = 0, 1, 2, \dots, N-1:
\quad y_{n+1} = y_n + h
\]

Since the slope is constant, each step adds \( h \) to the previous value.

After \( N \) steps:
\[
y_N = y_0 + N \cdot h = 0 + N \cdot \frac{1}{N} = 1
\]

\subsection{SDE Problem}

Consider the stochastic differential equation (SDE):
\[
dy = \,dt + 0.2\,dW_t, \quad y(0) = 0
\]
We want to simulate a numerical solution using the Euler-Maruyama method.

\subsubsection{Euler--Maruyama Method}

Given a time interval \([0, 1]\) and a number of steps \(N\), define the step size:
\[
h = \frac{1}{N}
\]
Let \(W_t\) denote the Wiener process. Then the Euler-Maruyama approximation is given by:
\[
y_{n+1} = y_n + h + 0.2 \Delta W_n
\]
where \(\Delta W_n \sim \mathcal{N}(0, h)\) is a normally distributed random variable with mean 0 and variance \(h\).

\subsection{Comparison}

\begin{table}[!h]
\begin{minipage}[t]{0.5\textwidth}
\centering
\begin{tabular}{llllc}
 \hline
Type & N & \# Operations & \# Fn Calls & Runtime (s) \\\hline
 ODE & 1 & 2 & 1 & 6.6e-5\\
 ODE & 10 & 20 & 10 & 8.0e-5\\
 ODE & 100 & 200 & 100 & 1.0e-4\\
 ODE & 1000 & 2000 & 1000 & 4.6e-4\\
 SDE & 10 & 40 & 20 & 1.2e-4\\
 SDE & 100 & 400 & 200 & 6.4e-4\\
 SDE & 1000 & 4000 & 2000 & 2.2e-3 \\\hline
\end{tabular}
\caption{Cost comparison for ODE/SDE solutions.}
\label{tab:ODEvSDE}
\end{minipage}
\hfill
\begin{minipage}[t]{0.5\textwidth}
\centering
\begin{tabular}{lccc}
 \hline
& \# Function  & \# Function & Total\\
& Calls  & Calls & Function\\
Type & for perturbation & for prediction & Calls\\ \hline
 Our & 1 & 1 & 2\\
 SOTA & 0 & 200 & 200\\ \hline
\end{tabular}
\caption{Cost comparison for generating one ensemble member using SOTA solution vs our solution for the case in \Cref{fig:summary}. }
\label{tab:ODEvSDE2}
\end{minipage}
\end{table}

\begin{figure}[!h]
        \centering
        \includegraphics[width=\linewidth]{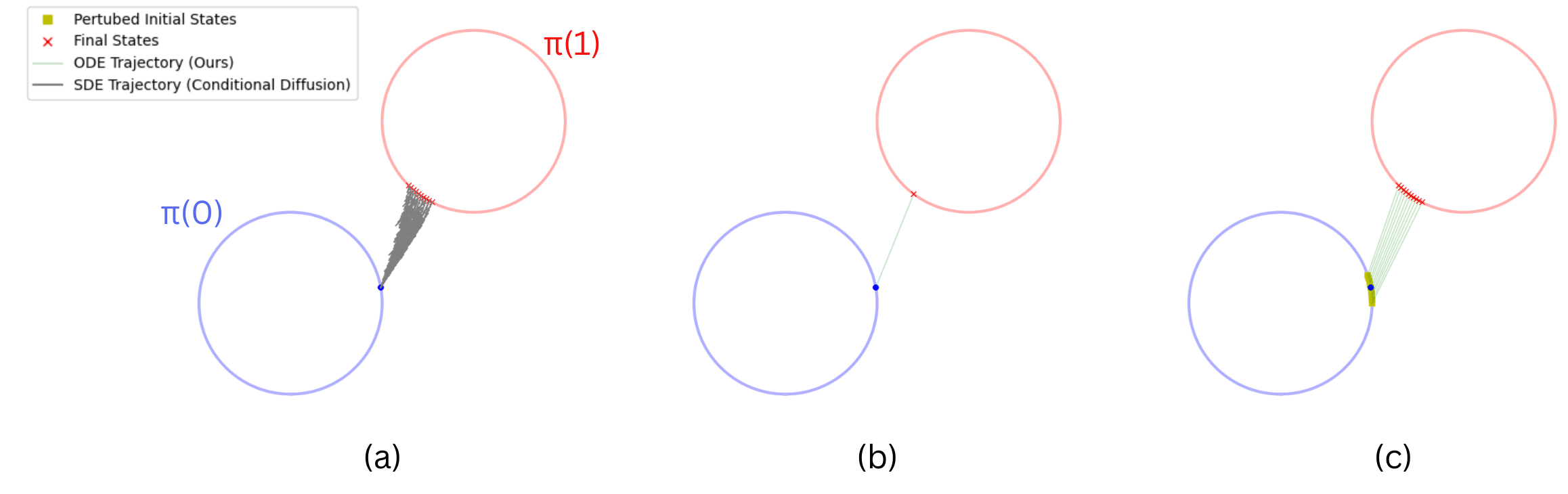}
        \caption{Flow trajectories for a state (blue dot) from the distribution density $\pi(0)$ to another distribution density $\pi(1)$ for inference of (a) five (5) member ensembles using SDE-based 200 step conditional diffusion model, (b) only deterministic prediction and (c) our physical-perturbation based five (5) member ensemble prediction using single step FM.}
        \label{fig:summary}
\end{figure}

As shown in \Cref{fig:ode}, the numerical solution is robust to the choice of discretization. From a theoretical perspective, for constant-slope ODEs such as our FM, using $N = 1$ is sufficient to obtain an accurate solution. In contrast, \Cref{fig:sde} shows that the solutions obtained with different step sizes vary significantly. Additionally, the error accumulation for the numerical integration of ODEs is reduced \cite{liu2022flow}. In practice, $N = 100$ is the simplest setting for solving SDEs. The typical order of $N$ is in thousands. As shown in \Cref{tab:ODEvSDE}, the method achieves a speedup of up to $30\times$. Therefore, at the scale of high-dimensional operations and large function evaluations, a speed of $\cal{O}$(10) is easily viable.

\clearpage
\section{Comparison Metrics} \label{app:metrics} 

\subsection{Continuous Ranked Probability Score (CRPS)}

\begin{equation}
\mathrm{CRPS} (y, x_1,\dots, x_M)
=
\frac{1}{|\Omega|}
\sum_{s \in \Omega}
\left[
\frac{1}{M} \sum_{j=1}^{M} |x_j(s) - y(s)|
-
\frac{1}{2M^2} \sum_{j=1}^{M} \sum_{k=1}^{M} |x_j(s) - x_k(s)|
\right]
\end{equation}

where
\begin{itemize}
    \item $\Omega$ denotes the spatial domain (e.g., grid points/pixels),
    \item $|\Omega|$ is the number of spatial locations,
    \item $s \in \Omega$ indexes a spatial location,
    \item $y(s)$ is the observed value at location $s$,
    \item $x_j(s)$ is the $j$-th ensemble member at location $s$,
    \item $M$ is the ensemble size.
\end{itemize}

\subsection{Ensemble Mean and Ensemble Standard Deviation}  
Let $S = {I_1,..,I_N}$ be a set of images, $I_i \in \cal{R}^{C\times H \times W}$ such that $i,j,k,N,C,H,W \in \mathbb{N}$, $i\leq N$, $j\leq C$, $k\leq H$, and $ l \leq W$. 
An image $I_i$ is defined as 
\[
I_i = \{ P_i^{j,k,l} \in \mathbb{R} ~|~ \text{j} \leq C, \text{k} \leq H, \text{l} \leq W\}, 
\]
where $P_i^{j,k,l}$ is called a pixel in an image $I_i$.\\

Ensemble mean state (image), $I_{EM} \in \mathbb{R}^{C\times H \times W }$, is defined as

\begin{equation}
I_{EM} = \{ P_{EM}^{j,k,l} = \frac{1}{N} \sum_{i=1}^{N} P_i^{j,k,l} | P_i^{j,k,l} \in I_i\}.
\end{equation}

Ensemble mean score $V_{EM}$ is defined as
\begin{equation}
V_{EM} = \frac{1}{N\cdot C \cdot H \cdot W} \sum_{i=1}^{N} \sum_{j=1}^{C} \sum_{k=1}^{H} \sum_{l=1}^{W} P_i^{j,k,l},
\end{equation}
where $ P_i^{j,k,l} \in I_i$.\\

Ensemble standard deviation state (image), $I_{ES} \in \mathbb{R}^{C\times H \times W }$, is defined as
\begin{multline}
I_{ES} = \Biggl\{ P_{ES}^{j,k,l} = \sqrt{\frac{1}{N} \sum_{i=1}^{N} (P_i^{j,k,l} - P_{EM}^{j,k,l})^2} \\ \Bigg| P_i^{j,k,l} \in I_i, P_{EM}^{j,k,l} \in I_{EM} \Biggl\}.
\end{multline}

Ensemble variance score $V_{ES}$ is defined as
\begin{equation}
V_{ES} = \frac{1}{N\cdot C \cdot H \cdot W} \sum_{i=1}^{N} \sum_{j=1}^{C} \sum_{k=1}^{H} \sum_{l=1}^{W} (P_i^{j,k,l} - V_{EM})^2,
\end{equation}
where $ P_i^{j,k,l} \in I_i$.

\subsection{MSE, MAE, SSIM}

\begin{equation}
\text{MSE} = \frac{1}{N\cdot C \cdot H \cdot W} \sum_{i=1}^{N} \sum_{h=1}^{H} \sum_{w=1}^{W} \sum_{c=1}^{C} (y - \hat{y})^2
\end{equation}

\begin{equation}
\text{MAE} = \frac{1}{N\cdot C \cdot H \cdot W} \sum_{i=1}^{N} \sum_{h=1}^{H} \sum_{w=1}^{W} \sum_{c=1}^{C} |y - \hat{y}|
\end{equation}

\begin{equation}
\text{SSIM(x,y)} = \frac{(2 \mu_x \mu_y + C_1)(2 \sigma_{xy} + C_2)}{(\mu_x^2 + \mu_y^2 + C_1)(\sigma_x^2 + \sigma_y^2 + C_2)}  
\end{equation}
\begin{equation}
\overline{\text{SSIM}} = \frac{1}{N} \sum_{i=1}^{N} \text{SSIM}(x,y)
\end{equation}

where:
\begin{align*}
N & \text{ is the number of images in the dataset,} \\
H & \text{ is the height of the images,} \\
W & \text{ is the width of the images,} \\
C & \text{ is the number of channels (e.g., 3 for RGB images),} \\
y & \text{ is the true pixel value at position } (i, h, w, c), \text{ and} \\
\hat{y} & \text{ is the predicted pixel value at position } (i, h, w, c).\\
\text{MAX} & \text{ is the maximum possible pixel value of the image,} \\
 & \text{(e.g., 255 for an 8-bit image),} \\
\text{MSE} & \text{ is the Mean Squared Error between the original}\\
 & \text{and compressed image.}\\
\mu_x & \text{ is the average of } x, \\
\mu_y & \text{ is the average of } y, \\
\sigma_x^2 & \text{ is the variance of } x, \\
\sigma_y^2 & \text{ is the variance of } y, \\
\sigma_{xy} & \text{ is the covariance of } x \text{ and } y, \\
C_1 & = (K_1 L)^2 \text{ and } C_2 = (K_2 L)^2 \text{ are two variable} \\
 & \text{to stabilize the division when the denominator is near zero,} \\
L & \text{ is the dynamic range of the pixel values} \\
 & \text{(typically, this is 255 for 8-bit images),} \\
K_1 & \text{ and } K_2 \text{ are small constants }\\
 & \text{(typically, } K_1 = 0.01 \text{ and } K_2 = 0.03).\\
\end{align*}

\newpage
\section{Dataset Configurations} \label{app:datasetstats}
\Cref{tab:datasetstats} describes the statistics, train-test splits, image sequences used for modelling, and frame resolutions.

\begin{table}[!h]
\centering
\begin{tabular}{lccccc}
\hline
{Dataset} & \( N_{\text{train}} \) & \( N_{\text{test}} \) & \( (C, H, W) \) & History & Prediction \\
\hline
Predator-Prey           & 10,000                  & 500                & (1, 1, 2)     &    1   & 1       \\
Vancouver93            & 9,935                  & 2,484                & (93, 1, 1)     & 5 (5 days)      & 5 (5 days)       \\
Moving MNIST           & 10,000                  & 10,000                & (1, 64, 64)     & 10      & 10       \\
CloudCast        & 52,560                  & 17,520                 & (1, 128, 128)   & 4 (1 Hr)       & 4 (1 Hr)       \\
WeatherBench (5.625$^{\circ}$)        & 324,336                  & 17,520                 & (48, 32, 64)   & 8 (48 Hrs)       & 8 (48 Hrs)       \\
\hline
\end{tabular}
\caption{Datasets statistics, training and testing splits, image sequences, and resolutions}
\label{tab:datasetstats}
\end{table}

\begin{table}[!h]
\centering
\begin{tabular}{clccc}
\hline
Type & Variable Name & Abbrev. & ECMWF ID  & Levels \\
\hline
 & Land-sea mask & LSM & 172 \\
Static & Orography & OROG & 228002\\
 & Soil Type & SLT & 43\\
\hline
 & 2 metre temperature & T2m & 167 \\
Single & 10 metre U wind component & U10 & 165 \\
 & 10 metre V wind component & V10 & 166 \\
\hline
 & Geopotential & Z & 129 & 50, 250, 500, 600, 700, 850, 925 \\
 & U wind component & U & 131 & 50, 250, 500, 600, 700, 850, 925 \\
Atmospheric & V wind component & V & 132 & 50, 250, 500, 600, 700, 850, 925 \\
 & Temperature & T & 130 & 50, 250, 500, 600, 700, 850, 925 \\
 & Specific humidity & Q & 133 & 50, 250, 500, 600, 700, 850, 925 \\
 & Relative humidity & R & 157 & 50, 250, 500, 600, 700, 850, 925 \\
\hline
\end{tabular}
\caption{Variables considered for global weather prediction model.}
\label{tab:wB_config}
\end{table}

\subsection{WeatherBench} \label{app:WBvars}
The original ERA5 dataset \cite{era5hersbach} has a resolution of 721$\times$1440, with hourly recordings spanning nearly 40 years from 1979 to 2018. The data cover 37 vertical levels on a 0.25$^\circ$ latitude–longitude grid. The raw data is too voluminous for efficient experimentation, even with powerful computing resources. We therefore use the typically used reduced resolutions (32$\times$64: 5.625$^\circ$ latitude-longitude grid, 128$\times$256: 5.625$^\circ$ latitude-longitude grid) as per the transformations made by  \cite{weatherBench}. We, however, stick to using the configuration set by \cite{nguyen2023climax} for standard comparison. The prediction model considers 6 atmospheric variables at 7 vertical levels, 3 surface variables, and 3 constant fields, resulting in 48 input channels for predicting four target variables: geopotential at 500hPa (Z500), temperature at 850hPa (T850) and 2 meters above ground (T2m), and zonal wind speed at 10 meters above ground (U10). These variables are standard in medium-range NWP models (e.g., Integrated Forecasting System \cite{wedi2015modelling}) and widely used as benchmarks in deep learning work \cite{graphcast2023,pathak2022fourcastnet}. We use a leap period of 6 hours as a single timestep, and hence our model takes in 8 timesteps (48 hours or 2 days) to predict for the next 8 timesteps (48 hours or 2 days). Under the same setup, we used data from 1979 to 2015 for training, 2016 for validation, and 2017–2018 for testing. The details of the variables  considered are in \Cref{tab:wB_config}.

\subsection{Additional Datasets} \label{app:addDatasets}
\begin{itemize}
    \item \textbf{Vancouver 93 Temperature Trend (Vancouver93)}: This is a real nonlinear chaotic high-dimensional dynamical system, in which only a single physical variable (temperature) is recorded.  The daily average temperatures at 93 different weather stations in and around Vancouver, BC, Canada, are captured for the past 34 years from sources such as the National Oceanic and Atmospheric Administration (NOAA), the Government of Canada, and Germany's national meteorological service (DWD) through Meteostat's Python library \cite{lamprechtmeteostat}.  Essentially, it is a time series of 12,419 records at 93 stations. The complex urban coastal area of Vancouver, with coasts, mountains, valleys, and islands, makes it a location where forecasting is notably more difficult than in general \cite{vannini2012making, leroyer2014subkilometer}. Historical temperature data is insufficient for prediction. Accurate forecasting requires high-resolution variables, including precipitation and pressure.
    
This makes the dataset well-suited for our proposed framework.
  \item \textbf{CloudCast} \cite{nielsen2020cloudcast}: a real physical nonlinear chaotic spatiotemporal dynamical system. The dataset comprises 70,080 satellite images at 128 $\times$ 128 pixel resolution (15 $\times$15 km), capturing 11 cloud types across multiple atmospheric levels. Images are pixel-level annotated and collected every 15 minutes from January 2017 to December 2018.
\end{itemize}

\subsubsection{Results} 

\begin{table}[t]
\centering
\begin{tabular}{lcccc}
\hline
 & {True}& {Our} & {True}& {Our}  \\
{Data}                 & {Mean}& {Mean} & {Std. Dev.}& {Std. Dev.}  \\\hline
Vancouver93            & 1.85e1 & 1.89e1      & 4.05  & 3.97     \\
CloudCast              & 1.53e-2 & 1.51e-2     & 9.84e-2  & 9.84e-2          \\
\hline
\end{tabular}
\caption{Comparison of mean and standard deviation of ensemble predictions.}
\label{tab:AddstatsMetrics}
\end{table}

\begin{table}[t]
\begin{minipage}[t]{0.49\textwidth}
\centering
\begin{tabular}{lcccc}
\hline
 {Data}                & {MSE}& {MAE} & {SSIM}\\\hline
Vancouver93            & 5.9e-1 & 6.1e-1          & NA  \\
CloudCast              & 2.0e-7 & 6.0e-4          & 0.99  \\
\hline
\end{tabular}
\caption{Similarity metrics for ensemble average.}
\label{tab:AddsimilarityMetricsMean}
\end{minipage}
\hfill
\begin{minipage}[t]{0.49\textwidth}
\centering
\begin{tabular}{lcccc}
\hline
 {Data}                & {MSE}& {MAE} & {SSIM}\\\hline
Vancouver93            & 1.8e-1 & 3.4e-1          & NA  \\
CloudCast              & 4.9e-7 & 6.0e-4          & 0.85  \\
\hline
\end{tabular}
\caption{Similarity metrics for ensemble standard deviation.}
\label{tab:AddsimilarityMetricsStd}
\end{minipage}
\end{table}

 \begin{figure}[b]
 \begin{minipage}[t]{0.49\textwidth}
    \centering
    \includegraphics[width=\linewidth]{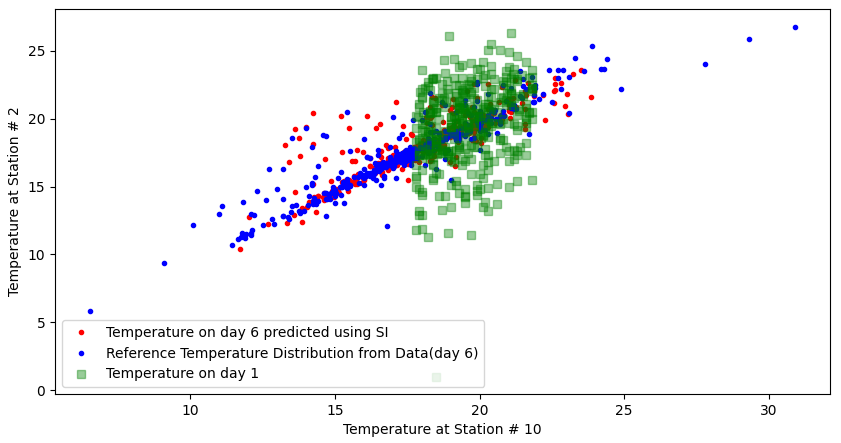}
    \caption{Phase diagram showing initial state and final states from data and model using FM after 5 days between station 2 and station 10.}
    \label{fig:unconditionalVancouver93_5days}
\end{minipage}
\hfill
\begin{minipage}[t]{0.49\textwidth}
    \centering
    \includegraphics[width=\linewidth]{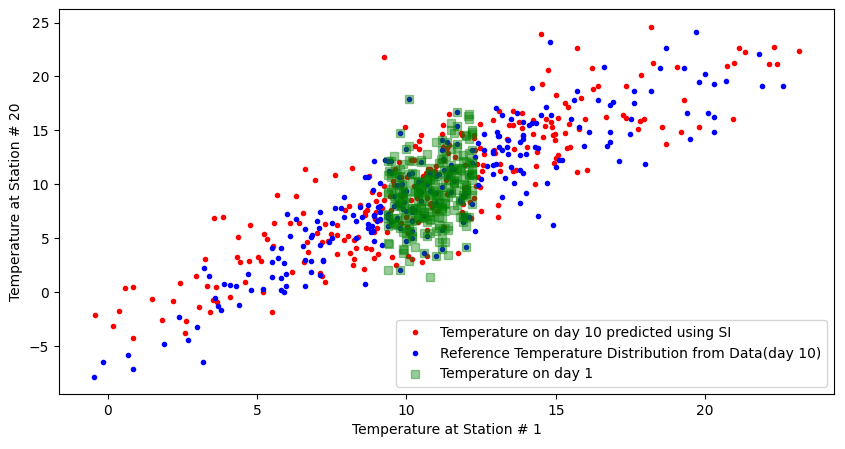}
    \caption{Phase diagram showing initial state and final states from data and model using FM after 10 days between station 20 and station 1.}
    \label{fig:unconditionalVancouver93_10days}
\end{minipage}
\end{figure}

\paragraph{Probabilistic Forecasting For Vancouver Temperature:} \label{Van93results}
We now use the Vancouver93 dataset to build a forecasting model using FM. The approximating function we use is a residual-network-based TCN inspired by \cite{bai2018empirical} along with time embedding on each layer. In this case, we propose a model that  uses a sequence of the last 5 days' temperatures to predict the next 5 days' temperatures.  For this problem, it is easy to visualize how FM learns the transportation from one distribution to another.

In the first experiment using this dataset, we select all test-set cases where the temperature at Station 1 is 20$^\circ$C and compare them with the empirical distribution of temperatures at the same stations 10 days later. Phase diagrams between any two stations reveal that final distributions from FM closely match those obtained from test data after 10 days. As shown in \Cref{fig:unconditionalVancouver93_5days}, these distributions exhibit strong skewness and closely align with the empirical state distribution at the 5-day lead time. Similarly, \Cref{fig:unconditionalVancouver93_10days} shows the distribution matching very well with the state after 10 days. Appendix \ref{app:van93stats} shows the matching histograms of the distributions for the latter case. Some metrics for comparison of results can be seen in \Cref{tab:AddstatsMetrics,tab:AddsimilarityMetricsMean,tab:AddsimilarityMetricsStd}.    

\begin{landscape}
\section{Perturbations using SI: Additional Study} \label{app:MM} 

\begin{figure*}[!h]
    \centering
    \includegraphics[width=\linewidth]{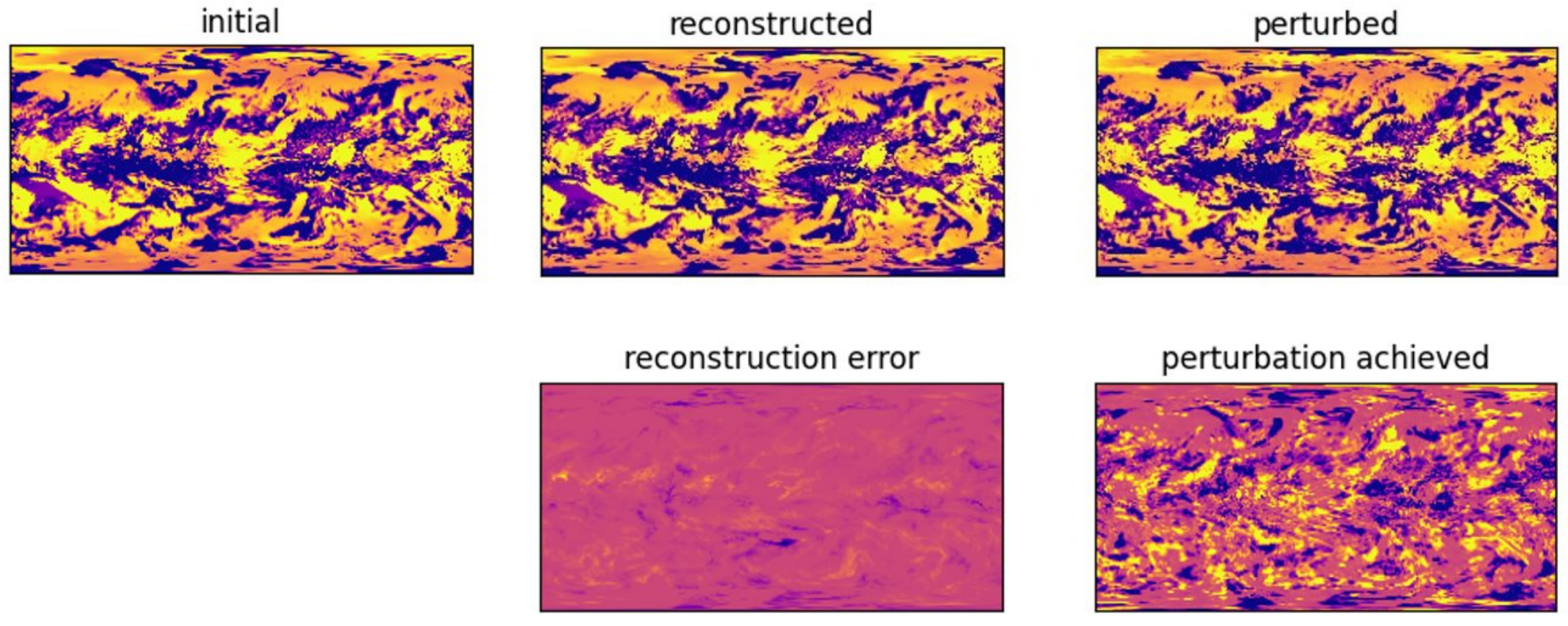}
    \caption{Physical perturbation of the clouds in the whole atmosphere at a single random altitude extracted from WeatherBench at 1.40625$^\circ$ resolution, using our technique.}
    \label{fig:wB140625}
\end{figure*}

\clearpage

\begin{figure}[!h]
\begin{minipage}[t]{0.49\linewidth}
    \centering
    \includegraphics[width=\linewidth]{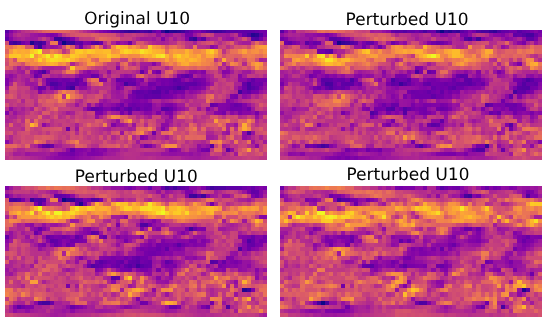}
    \caption{Perturbed states of a WeatherBench U10 state.}
    \label{fig:wBperturbedU10}
\end{minipage}
\hfill
\begin{minipage}[t]{0.49\linewidth}
    \centering
    \includegraphics[width=\linewidth]{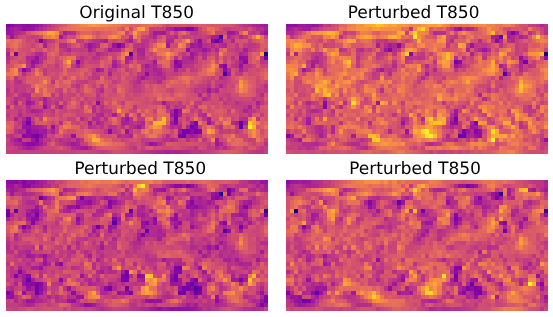}
    \caption{Perturbed states of a WeatherBench T850 state.}
    \label{fig:wBperturbedT850}
\end{minipage}
\end{figure}

\begin{figure}[!h]
    \centering
    \includegraphics[width=0.5\linewidth]{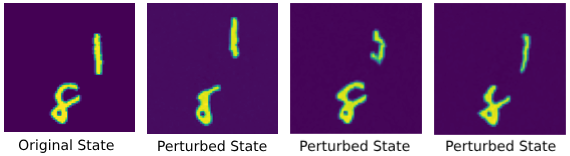}
    \caption{Perturbed states of a Moving MNIST sample state.}
    \label{fig:MMperturbed}
\end{figure}

\paragraph{Perturbation of Physical States using FM:}
\label{exPhysicalPerturbation}
Perturbation of a physical state in a dynamical system refers to a deliberate natural deviation in a system's parameters or state variables. Using the algorithm proposed in \Cref{sec:GeneratePerturbation}, we generate perturbed states for MovingMNIST and WeatherBench. All the generated perturbed states are physically realistic. Some sample perturbed states from MovingMNIST are shown in \Cref{fig:MMperturbed}. Physical disturbances such as digit bending, deformation, position changes, and shape variations are clearly captured. Similarly, we generated perturbed states for WeatherBench as well. Some samples of perturbed states from WeatherBench's U10 and T850 variables (as these are more interpretable) can be seen in \Cref{fig:wBperturbedU10,fig:wBperturbedT850}, respectively. As can be seen, the perturbations are very much physical.

\end{landscape}

\subsection{Ablation Study}\label{sec:ablation}

We further evaluate the model by applying non-Gaussian perturbations to the Gaussian latent state and examining the resulting physically plausible outputs. We consider two cases: one with perturbation amplitude $\sigma = 0.2$ and another with $\sigma = 0.5$. Reconstruction metrics indicate how well the model learns the transformation without any perturbation to the latent state. An effective perturbation introduces small variations while maintaining visual similarity to the input state. The model is then tested with three different types of perturbations of the latent state. Firstly, with a normal random noise, then with uniform random noise, and finally with a constant offset noise. \Cref{tab:diffperturbation} shows that normal perturbation works as expected; however, quantitative metrics alone are insufficient. The images provide clear evidence that the model responds specifically to Gaussian perturbations, consistent with our theoretical framework (see \Cref{fig:ablation_perturb}).  

\begin{table}[!b]
\begin{minipage}[!t]{0.48\linewidth}
            \centering
            \begin{tabular}{l|cccccc}
            \hline
            Type of Noise & Reconst. & Normal & Uniform & Constant \\
            \hline 
            $\sigma = 0.2$\\
            MSE & 4.5e-3 & 9.7e-3 & 4.2e-2 & 1.5e-1 \\
            SSIM & 0.942 & 0.871 & 0.154 & 0.057\\
            \hline 
            $\sigma = 0.5$\\
            MSE & 4.5e-3 & 2.8e-2 & 1.7e-1 & 1.0e0 \\
            SSIM & 0.942 & 0.246 & 0.067 & 0.025\\
            \hline
        \end{tabular}
        \caption{Influence of different perturbations on generating physical perturbations.}
        \label{tab:diffperturbation}
\end{minipage} \hfill
\begin{minipage}[!t]{0.48\linewidth}
    \centering
    \includegraphics[width=\linewidth]{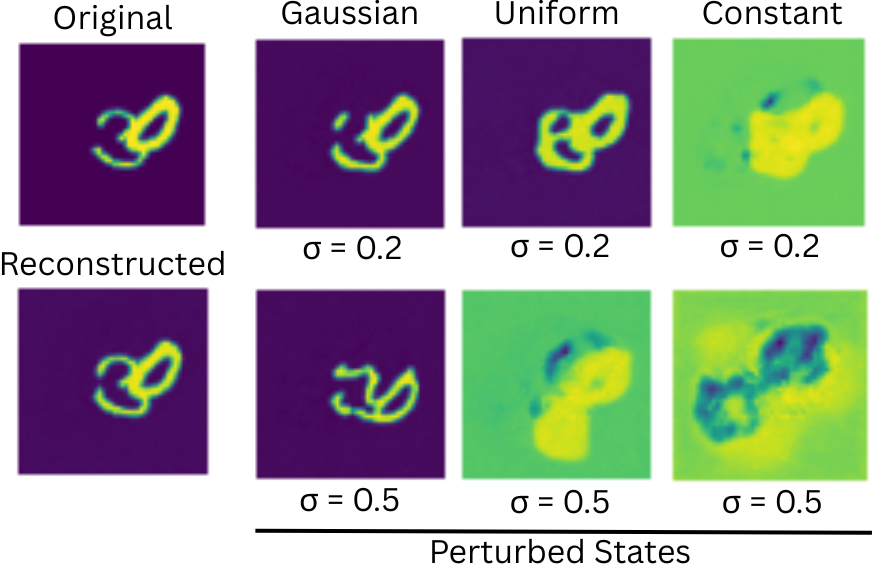}
    \caption{Perturbed states of a moving MNIST state by the addition of Gaussian, uniform, and constant noises. }
    \label{fig:ablation_perturb}
\end{minipage}
\end{table}

\begin{figure}[!b]
    \centering
    \includegraphics[width=\linewidth]{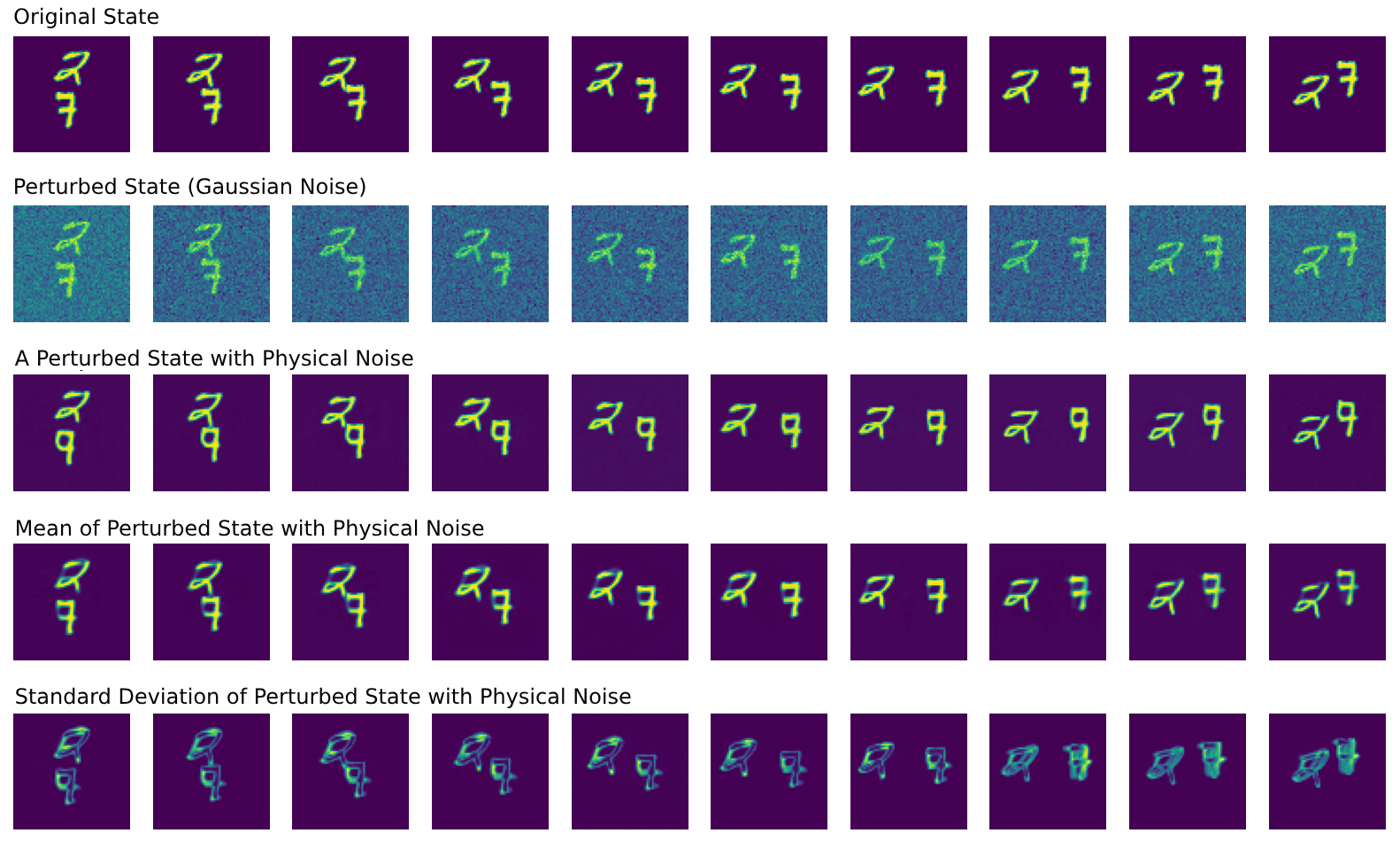}
    \caption{Statistical comparison of 100 perturbed states from a single MovingMNIST sequence using FM.}
    \label{fig:MMPertStats}
\end{figure}
        
\subsection{MovingMNIST} 
\Cref{fig:MMPertStats} shows how the perturbed state generated using FM is better than Gaussian perturbations. \Cref{fig:MMPertStats} illustrates an important phenomenon: under mild perturbations, the model can transition between digit classes, as shown in the third row. Despite this, under mild perturbations, the mean of 100 sampled trajectories closely matches the original state. The corresponding standard deviation captures the extent of the perturbations. \Cref{fig:MMPertvsNoise} shows how a perturbed state varies with different noise levels. With higher noise levels, the digits undergo topological changes, transitioning between classes (e.g., 'one' morphing into 'four'). The model captures the topological changes required for digits to transition into other classes.  

\begin{figure}[t]
    \centering
    \includegraphics[width=\linewidth]{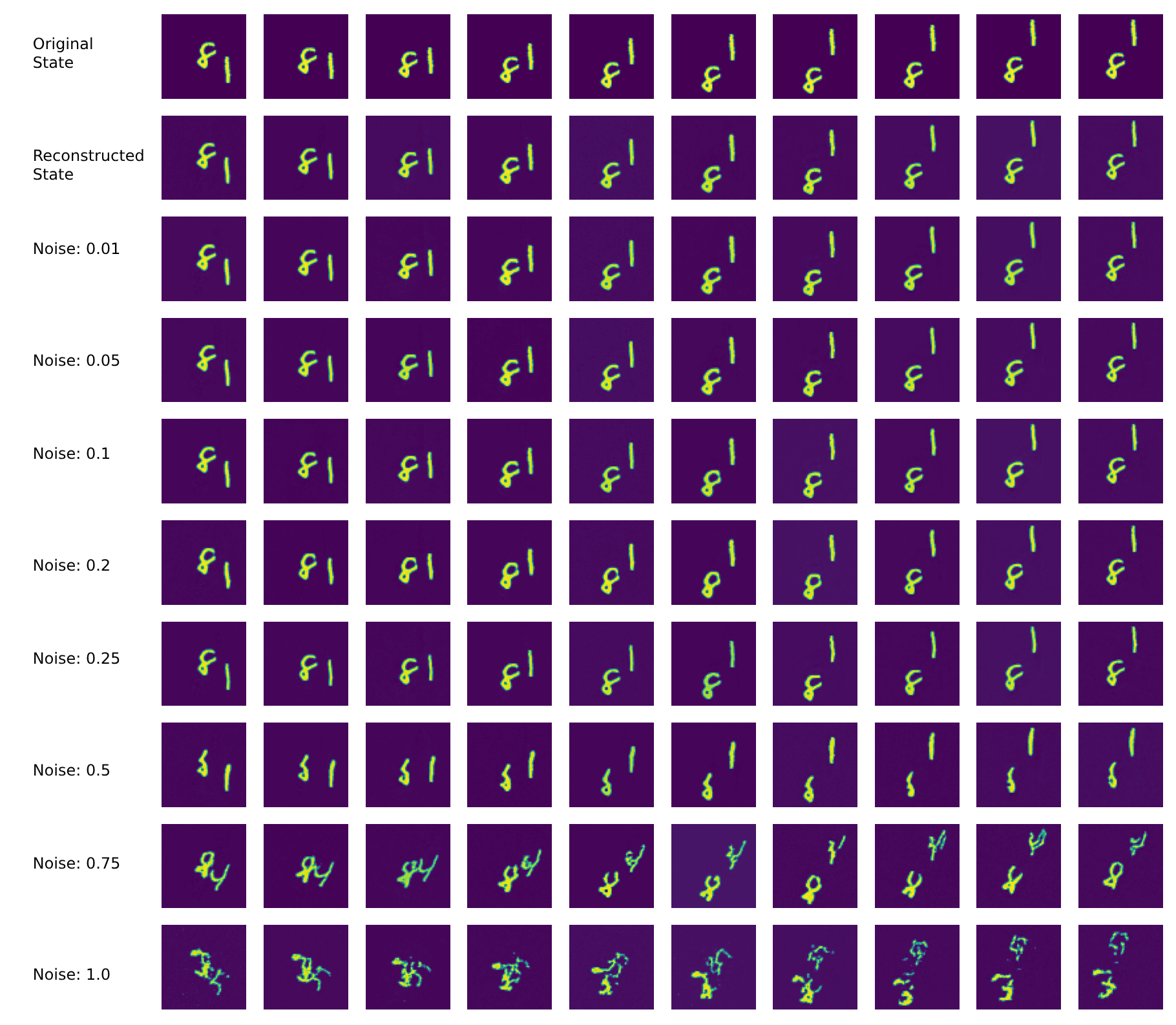}
    \caption{Perturbed states for a single MovingMNIST sequence using FM for different levels of noise.}
    \label{fig:MMPertvsNoise}
\end{figure}

\clearpage
\subsection{WeatherBench} 
\Cref{fig:wBPertvsNoise} shows the four curated important variables for weather prediction, Z500, T2m, U10, and T850, perturbed with different values of noise.

\begin{figure}[!b]
    \centering
    \includegraphics[width=\linewidth]{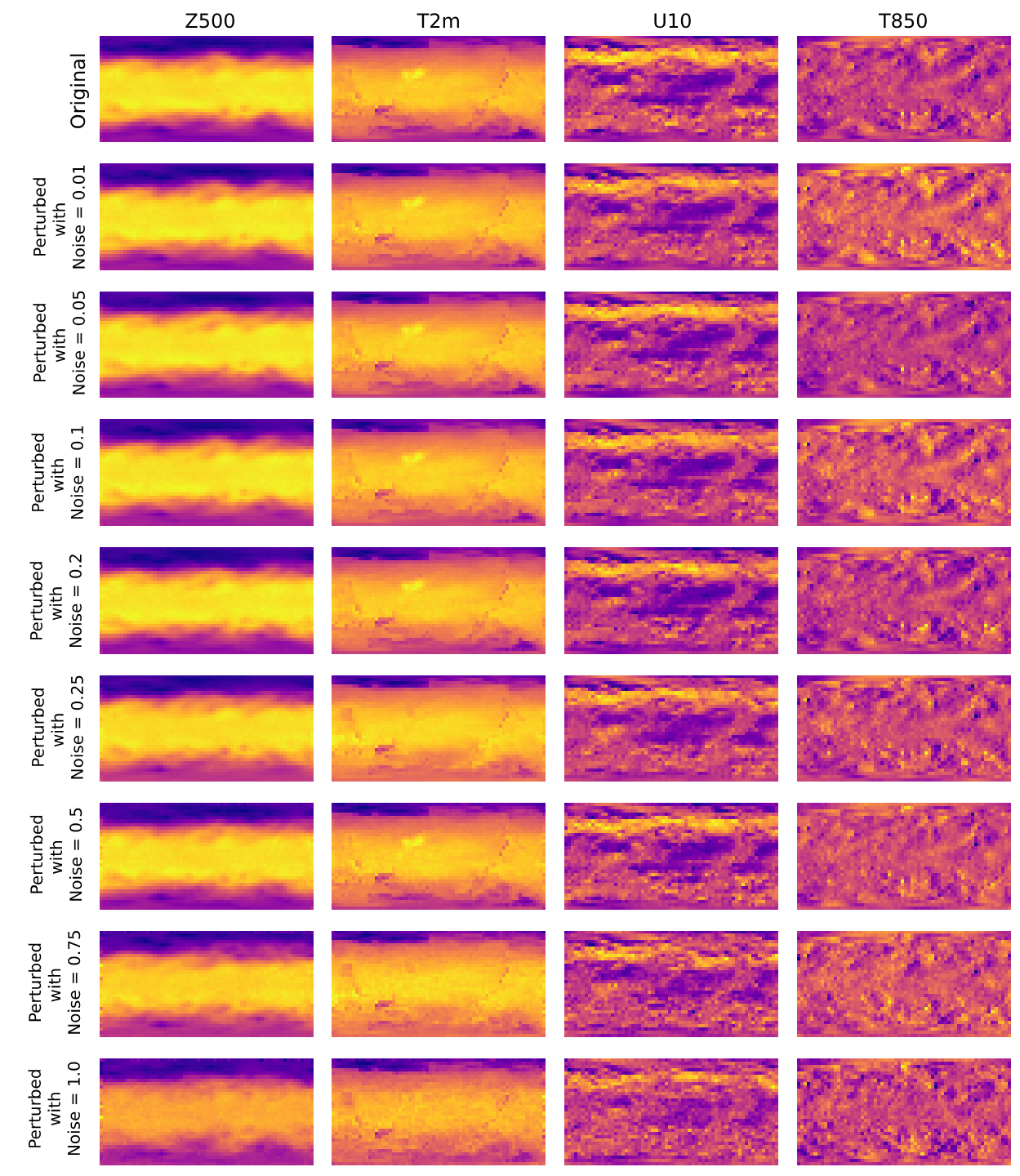}
    \caption{Perturbed states for a WeatherBench state using FM for different levels of noise.}
    \label{fig:wBPertvsNoise}
\end{figure}

\clearpage
\section{Statistical Comparison of Our Ensemble Predictions} 
\subsection{Predator-Prey Model} \label{app:predpreystats}
\Cref{fig:hist_predprey} shows the histograms of $\bfy_1$ and $\bfy_2$ respectively. The histogram of the actual state variable closely matches the histogram from FM predictions.

\begin{figure}[!h]
        \centering
        \begin{tabular}{cc}
        \includegraphics[width=0.50\linewidth]{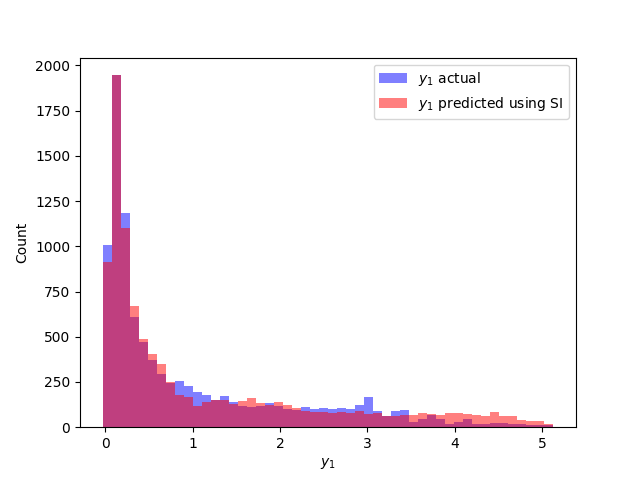} &
        \includegraphics[width=0.52\linewidth]{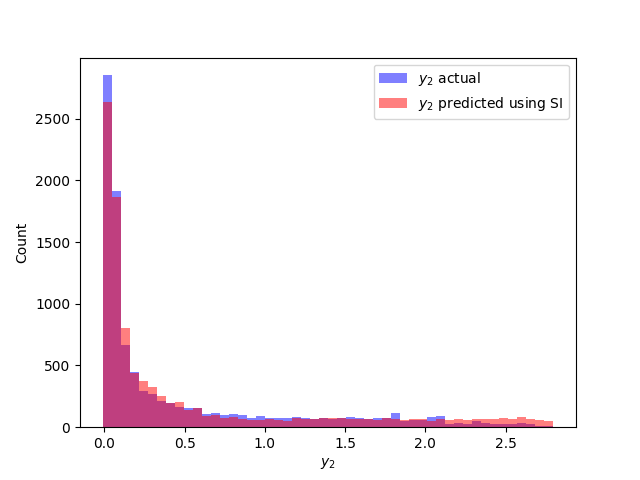}
        \end{tabular}
        \caption{Histograms of actual final distribution of $\bfy_1$ (left) and $\bfy_2$ (right) compared with that obtained using FM on the predator-prey model}
        \label{fig:hist_predprey}
\end{figure}

\subsection{Vacouver93} \label{app:van93stats}
\Cref{fig:unconditionalVancouver93station1Hist} shows the histograms of the two distributions, which are very similar, suggesting that FM efficiently learns the transport map to the distribution of temperatures at a station for this problem, whose deterministic solution is computationally demanding. 

\begin{figure}[!b]
    \centering
    \includegraphics[width=0.8\linewidth]{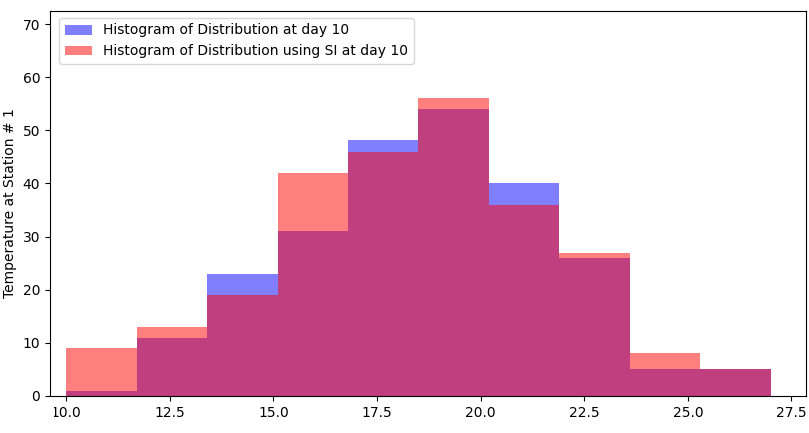}
    \caption{Histograms of distributions observed after 10 days and distribution obtained using FM excluding outliers.}
    \label{fig:unconditionalVancouver93station1Hist}
\end{figure}

\subsection{Moving MNIST} \label{app:MMstats}

\begin{figure}[!b]
    \centering
    \includegraphics[width=1.0\linewidth]{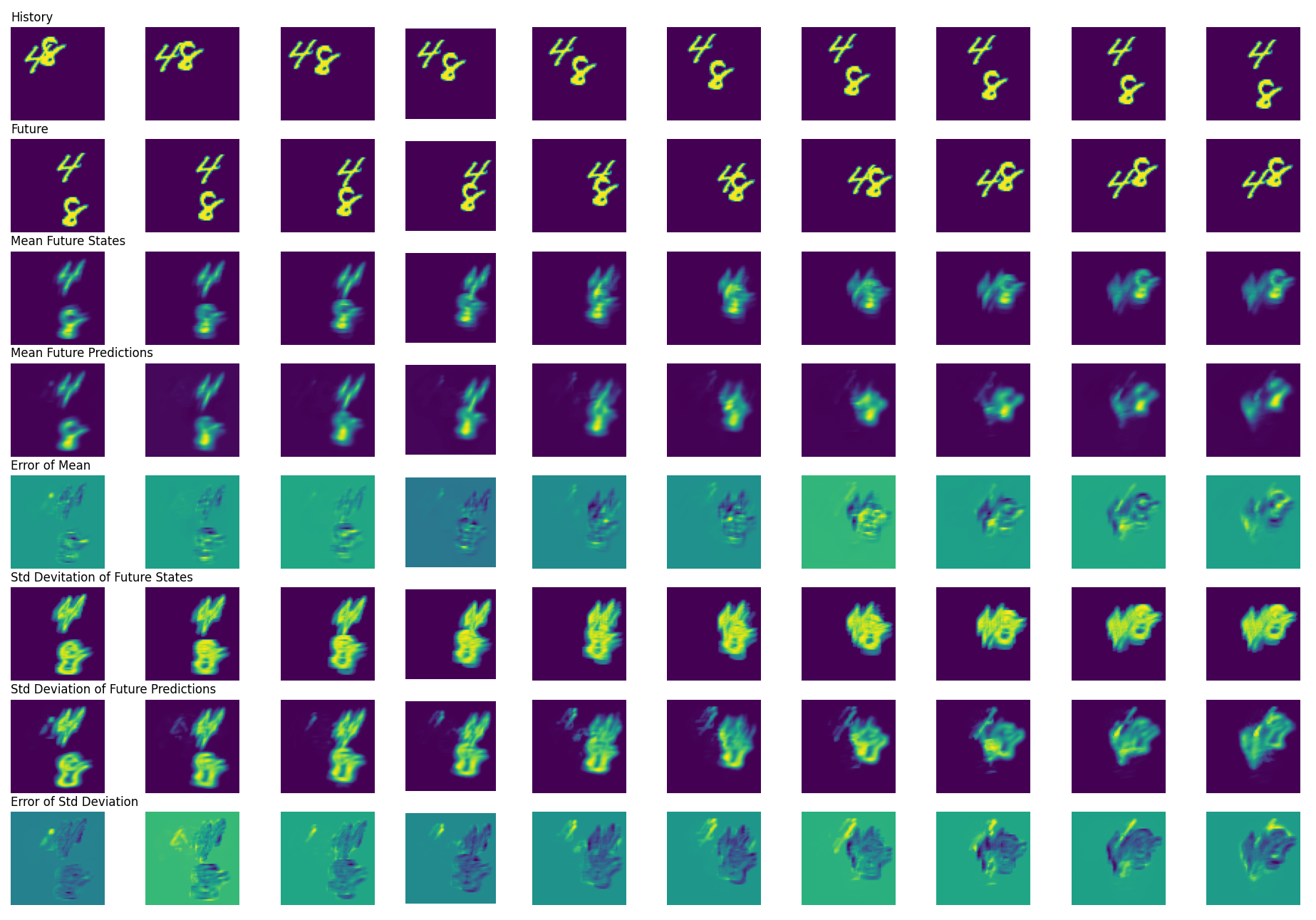}
    \caption{Statistical comparison of Moving MNIST sequences predicted using FM with random perturbed initial states.}
    \label{fig:MMLowPert}
\end{figure}

\begin{figure}[!t]
    \centering
    \includegraphics[width=1.0\linewidth]{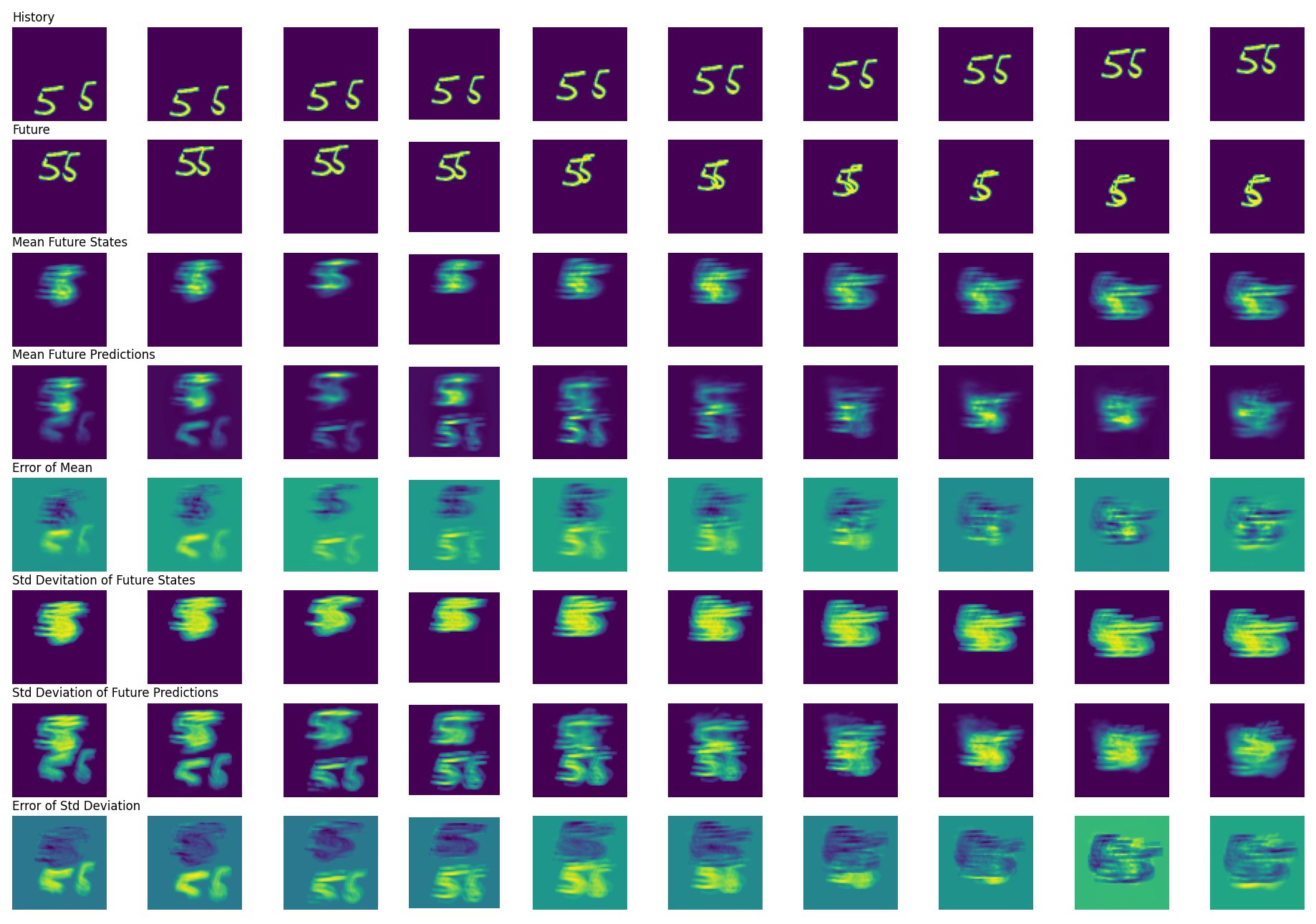}
    \caption{Statistical comparison of Moving MNIST sequence with random large perturbed initial states.}
    \label{fig:MMLargePert}
\end{figure}

 Consider the initial 10 frames from a case in the Moving MNIST test set as history, and the actual 10 subsequent frames in the sequence as future. For such a case, \Cref{fig:MMLowPert,fig:MMLargePert} show statistical images for small perturbation and large perturbation, respectively. 100 perturbed samples are taken, and their final states are predicted using SI. The ensemble mean and standard deviation closely align with the ground truth, as shown in the figures. This validates that the model effectively captures the uncertainty inherent in the Moving MNIST dataset. 

\clearpage
\subsection{Cloudcast}\label{app:CCstats}
\subsubsection{Predictions}
We use 4-timeframe sequences from the CloudCast test set to predict the next 4 timeframes at 15-minute intervals. \Cref{fig:cloudcastensembleupdatepredictions} shows twelve ensemble predictions. While predictions are nearly identical at a 1-hour lead time, fine-scale differences become apparent at higher resolution. \Cref{fig:cloudcastensemblepredictions} demonstrates that these differences amplify significantly over longer horizons: at a 3-hour lead time, cloud patches show distinct differences in shape and size. This sensitivity reflects the chaotic nature of cloud dynamics, where small imperceptible differences at early times can produce substantially different outcomes within hours.   

\begin{figure*}[!b]
    \centering
    \includegraphics[width=1.0\linewidth]{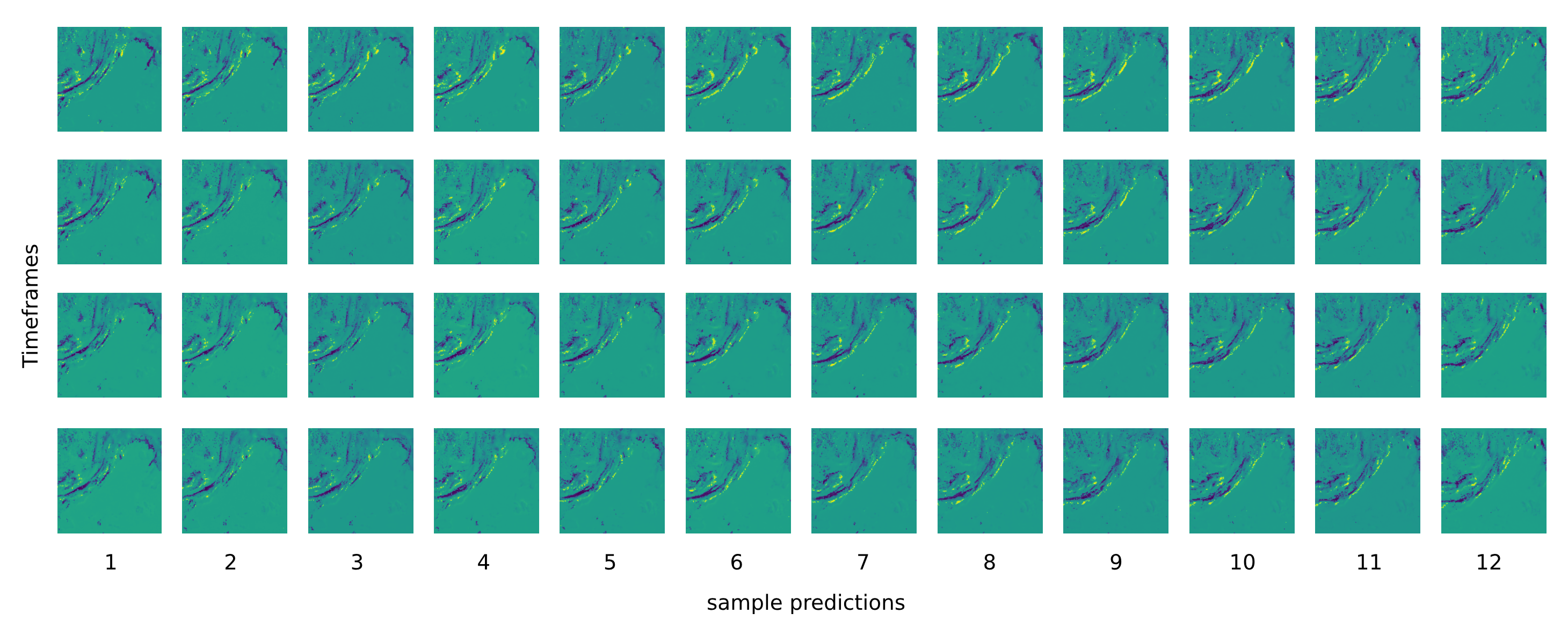}
    \caption{Random predictions of differential change from a single sequence from the cloudcast testset.}
    \label{fig:cloudcastensembleupdatepredictions}
\end{figure*}

\begin{figure*}[!b]
    \centering
    \includegraphics[width=1.0\linewidth]{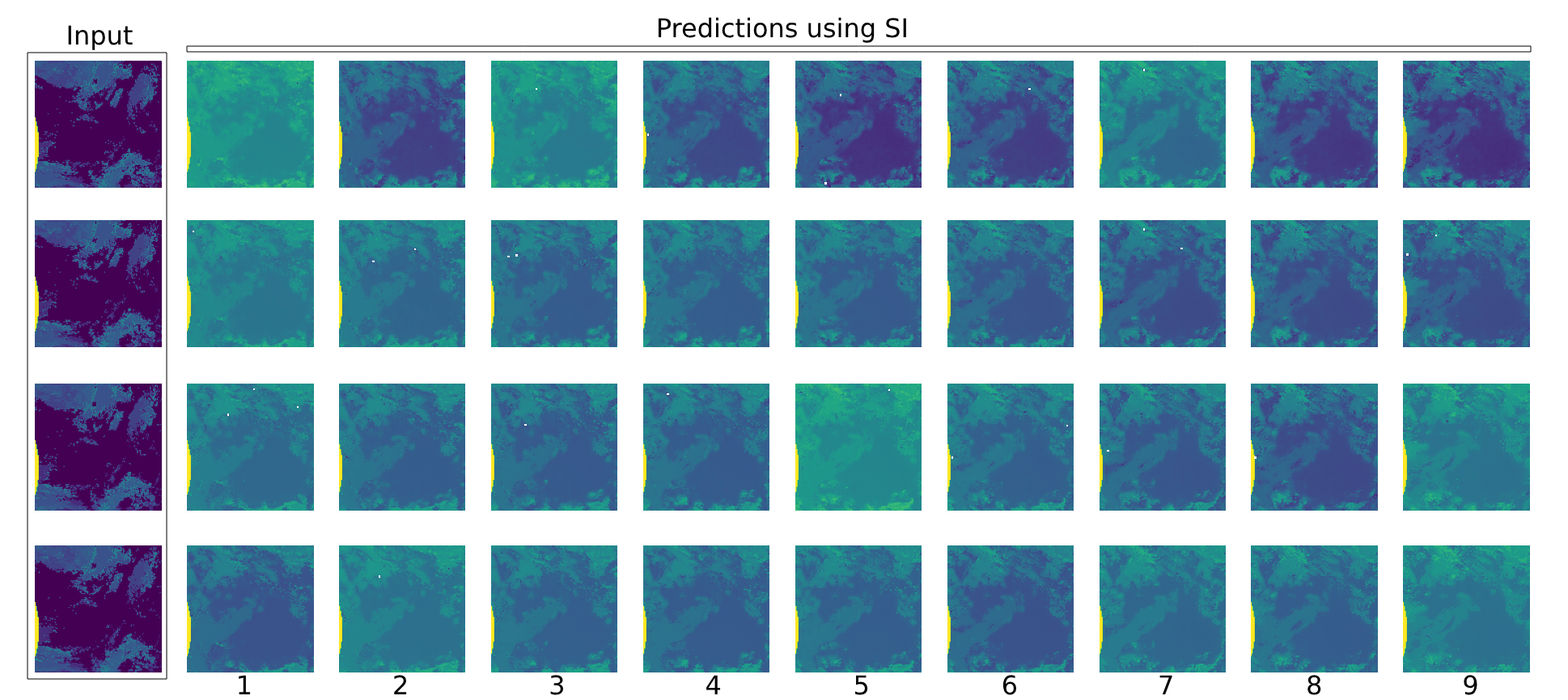}
    \caption{Random predictions from a single sequence from CloudCast testset for prediction after 3 hours.}
    \label{fig:cloudcastensemblepredictions}
\end{figure*}

\subsubsection{Statistical comparison}
\Cref{fig:CCPertStats} shows statistical images for generated outputs on the CloudCast testset using SI. Similar samples are taken as the set of initial states to predict their final states using SI. The figures demonstrate that the ensemble mean and standard deviation closely match the ground truth. 

\begin{figure}[!b]
    \centering
    \includegraphics[width=0.76\linewidth]{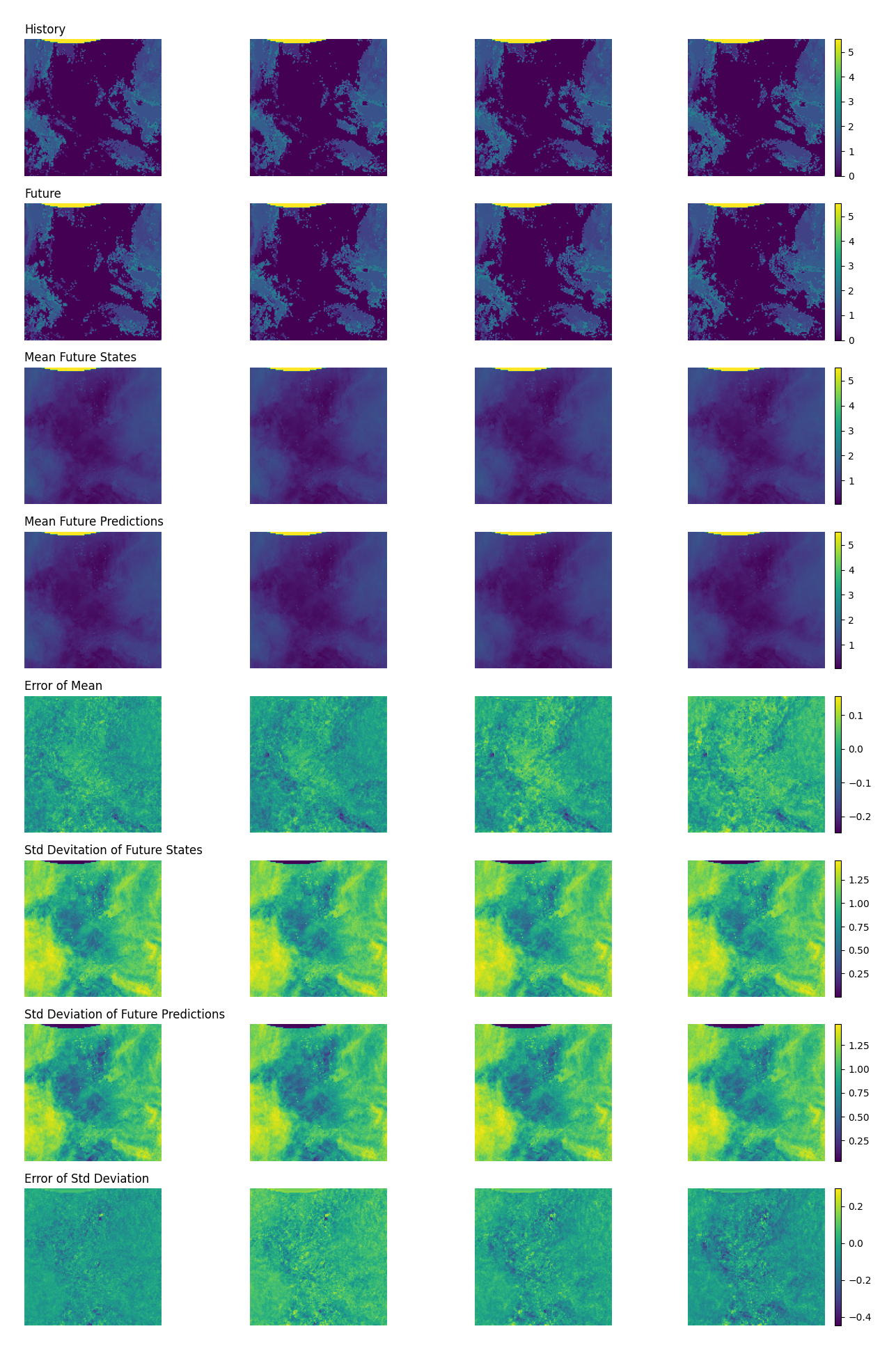}
    \caption{Statistical comparison of CloudCast sequences predicted using FM for a single initial state.}
    \label{fig:CCPertStats}
\end{figure}

\clearpage
\begin{landscape}
\subsection{WeatherBench}\label{app:wBstats}
\subsubsection{Predictions}
\Cref{fig:wBsample20forecasts} shows some of the predictions using FM with mild perturbations.

\begin{figure}[!h]
    \centering
    \includegraphics[width=1.05\linewidth]{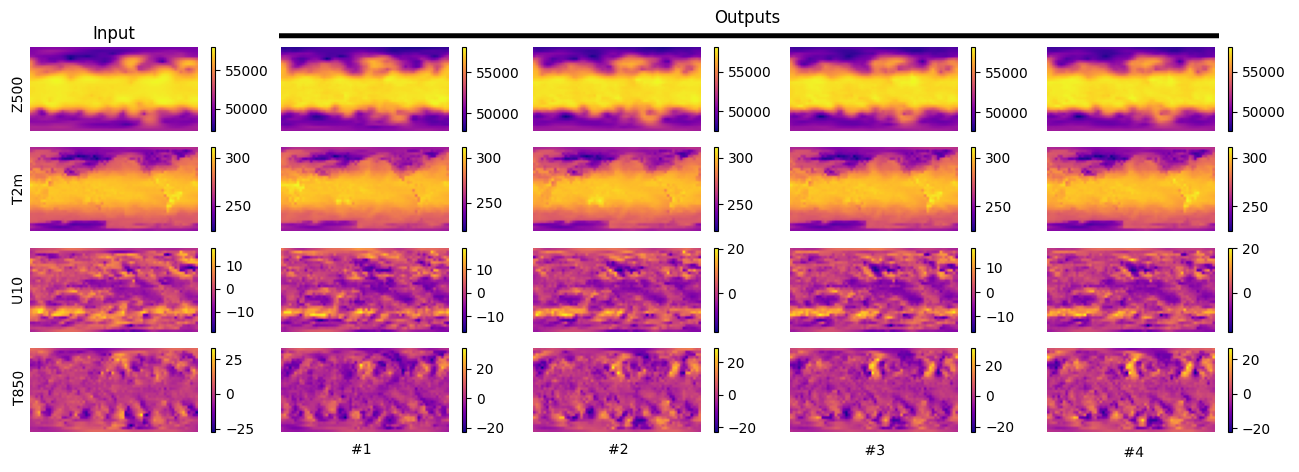}
    \caption{Four sample stochastic forecasts of Z500, T2m, U10 and T850 after 2 days obtained using SI.}
    \label{fig:wBsample20forecasts}
\end{figure}

\clearpage

\subsubsection{Statistical comparison}
\Cref{fig:wBstats20forecasts} shows the statistical comparison of ensemble predictions with 20 mildly perturbed states. \Cref{fig:wBstats78forecasts} shows the statistical comparison of ensemble predictions with 78 strongly perturbed states. The mean and standard deviation of the states for 2 days (48-hour predictions) can be easily compared. \Cref{tab:similarityMetricsWB6hrs,tab:similarityMetricsWB2days} show the metrics to compare the accuracy of our ensemble mean and ensemble standard deviation for 6-hour and 2-day predictions, respectively. 

\begin{figure}[!b]
    \centering
    \includegraphics[width=1.05\linewidth]{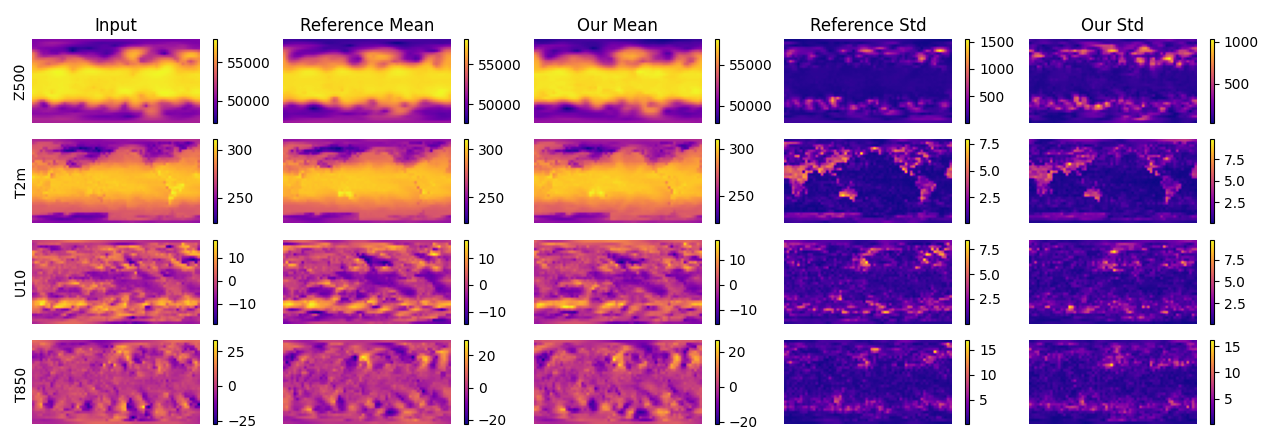}
    \caption{Statistical comparisons of Z500, T2m, U10, and T850 for 2-day ensemble forecasting using FM using 20 mildly perturbed samples.}
    \label{fig:wBstats20forecasts}
\end{figure}

\begin{figure}[t]
    \centering
    \includegraphics[width=1.05\linewidth]{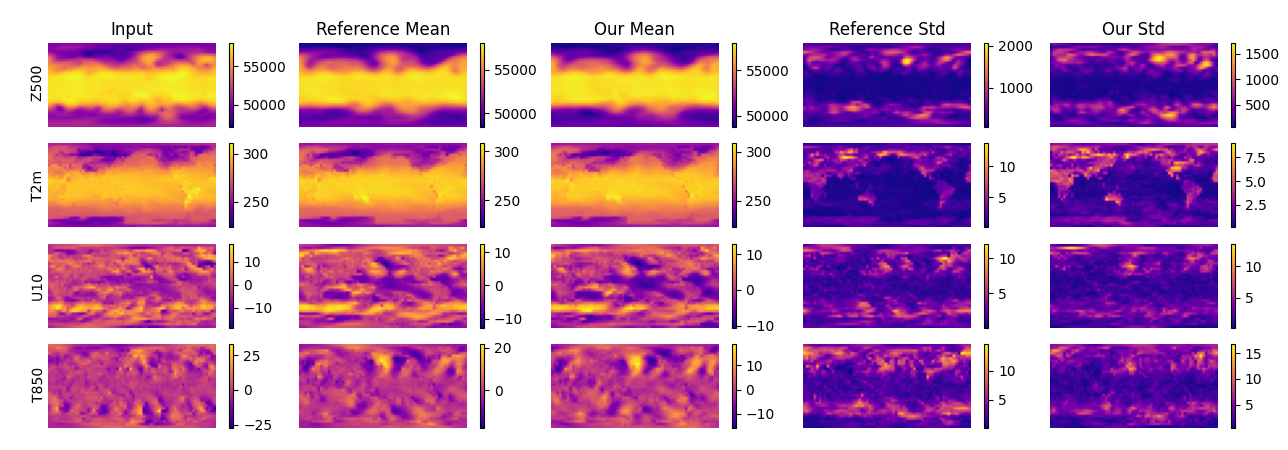}
    \caption{Statistical comparisons of Z500, T2m, U10, and T850 for 2-day ensemble forecasting using FM using 78 strongly perturbed samples.}
    \label{fig:wBstats78forecasts}
\end{figure}

\begin{table}[t]
\centering
\begin{tabular}{l|c|ccccc|ccccc}
\hline
& & \multicolumn{5}{c}{Ensemble Mean} & \multicolumn{5}{|c}{Ensemble Std. Dev.} \\ \cline{3-12}
{Variable} & CRPS &   True Score    &   Our Score    &   MSE($\downarrow$)   &   MAE($\downarrow$)   &   SSIM($\uparrow$)   &   True Score  &  Our Score     &   MSE($\downarrow$)    &   MAE($\downarrow$)    &   SSIM($\uparrow$)    \\ \hline
Z500       & 6.98e1 & 5.40e4   &  5.40e4  &  2.00e4 &  1.04e2    &  0.992    & 3.96e2 & 3.74e2 &  1.83e4
&  9.16e1    &  0.850 \\
T2m        & 7.51e-1 &   2.78e2  &  2.78e2  &  2.62   &  9.33e-1    &   0.986  & 1.78 & 1.84   &       8.90e-1 &  5.60e-1   &  0.745\\
U10        & 6.93e-1 &  -1.85e-1 & -2.70e-1 &  1.27   &  8.40e-1   &    0.886  & 2.51  & 2.34  &        1.12   & 7.35e-1 &  0.708 \\
T850       & 7.02e-1 &  -2.43e-2  & -3.74e-2 &  1.62  &  9.79e-1  &    0.812  & 3.75   & 3.58   &       1.29  &  8.26e-1    & 0.774\\
\hline
\end{tabular}
\caption{Similarity metrics for weatherBench ensemble prediction after 6 hours.}
\label{tab:similarityMetricsWB6hrs}
\end{table}

\begin{table}[t]
\centering
\begin{tabular}{l|c|ccccc|ccccc}
\hline
& & \multicolumn{5}{c}{Ensemble Mean} & \multicolumn{5}{|c}{Ensemble Std. Dev.} \\ \cline{3-12}
{Variable} & CRPS &   True Score    &   Our Score    &   MSE($\downarrow$)   &   MAE($\downarrow$)   &   SSIM($\uparrow$)   &   True Score  &  Our Score     &   MSE($\downarrow$)    &   MAE($\downarrow$)    &   SSIM($\uparrow$)    \\ \hline
Z500       & 1.57e2 &  5.40e4   &  5.40e4  &  7.86e4 &  2.04e2    &  0.978    & 3.58e2 & 3.78e2 &       4.52e4 &  1.41e2    &  0.630\\
T2m        & 7.10e-1 &   2.78e2  &  2.78e2  &  3.05   &  1.06    &    0.986  & 1.83   & 1.95   &       1.54   &  7.3e-1    &  0.681\\
U10        & 7.95e-1 &  -1.16e-1 & -1.10e-1 &  2.42   &  1.17    &    0.820  & 2.42   & 2.58   &        2.19   & 1.04    &  0.561\\
T850       & 9.08e-1 & -8.09e-2  & -8.12e-2 &  4.30   &  1.53    &    0.688 & 3.64   & 3.72   &       3.12   &  1.28    & 0.561\\
\hline
\end{tabular}
\caption{Similarity metrics for weatherBench ensemble prediction after 2 days.}
\label{tab:similarityMetricsWB2days}
\end{table}

\end{landscape}

\clearpage

\section{Hyperparameter Settings and Computational Resources}
\subsection{UNet Training}
\Cref{tab:hyperparameters_mnistCC} shows the hyperparameter settings for training on MovingMNIST and CloudCast datasets. While \Cref{tab:hyperparameters_wB} shows the settings for training on WeatherBench.

\label{UNetsettings}
\begin{table}[!h]
\centering
\begin{tabular}{lcccc}
\hline
{Hyperparameter}     & {Symbol} & {Value}  \\ \hline
Learning Rate               & $\eta$          & $1e-4$         \\
Batch Size                  & $B$             & $64$           \\
Number of Epochs            & $N$             & $200$           \\
Optimizer                   & -               & Adam            \\
Dropout                     & -               & 0.1             \\
Number of Attention Heads   & -               & 4             \\
Number of Residual Blocks   & -               & 2             \\
\hline
\end{tabular}
\caption{Neural Network Hyperparameters for training Moving MNIST and CloudCast}
\label{tab:hyperparameters_mnistCC}
\end{table}

\label{UNetsettingswB}
\begin{table}[!h]
\centering
\begin{tabular}{lcccc}
\hline
{Hyperparameter}     & {Symbol} & {Value}  \\ \hline
Learning Rate               & $\eta$          & $1e-4$         \\
Batch Size                  & $B$             & $8$           \\
Number of Epochs            & $N$             & $50$           \\
Optimizer                   & -               & Adam            \\
Dropout                     & -               & 0.1             \\
Number of Attention Heads   & -               & 4             \\
Number of Residual Blocks   & -               & 2             \\
\hline
\end{tabular}
\caption{Neural Network Hyperparameters for training WeatherBench}
\label{tab:hyperparameters_wB}
\end{table}

\subsection{Computational Resources} 
\label{resources}
All our experiments are conducted using an NVIDIA RTX-A6000 GPU with 48GB of memory. 

\subsection{Software Resources} 
\label{resources}
All our experiments are conducted using Python 3.10.12 with CUDA-enabled PyTorch library running on the Ubuntu 20.04 LTS operating system.

\end{document}